\documentclass{article}
\PassOptionsToPackage{numbers,compress}{natbib}
\usepackage[main,final]{neurips_2026}

\usepackage{amsmath,amssymb,amsfonts,amsthm,mathtools}

\usepackage[utf8]{inputenc}
\usepackage[T1]{fontenc}
\usepackage{hyperref}
\usepackage{url}
\usepackage{xcolor}
\usepackage{booktabs}
\usepackage{nicefrac}
\usepackage{microtype}
\usepackage{graphicx}
\usepackage{subcaption}
\usepackage{multirow}
\usepackage{wrapfig}

\usepackage{listings}
\definecolor{lstbg}{HTML}{F7F7F9}
\definecolor{lstkw}{HTML}{AA3731}
\definecolor{lststr}{HTML}{448C27}
\definecolor{lstcom}{HTML}{7D8B99}
\definecolor{lstnum}{HTML}{7A3E9D}
\lstdefinestyle{pythonabc}{
  language=Python,
  basicstyle=\ttfamily\footnotesize,
  keywordstyle=\color{lstkw}\bfseries,
  stringstyle=\color{lststr},
  commentstyle=\color{lstcom}\itshape,
  numberstyle=\tiny\color{lstcom},
  backgroundcolor=\color{lstbg},
  frame=single,
  framerule=0pt,
  rulecolor=\color{lstbg},
  xleftmargin=6pt,
  xrightmargin=6pt,
  aboveskip=6pt,
  belowskip=6pt,
  breaklines=true,
  showstringspaces=false,
  columns=fullflexible,
  keepspaces=true,
}
\newcommand{\dnred}[1]{\textcolor{red!70!black}{\raisebox{0.3ex}{\scriptsize$\downarrow$\,#1}}}
\newcommand{\upred}[1]{\textcolor{red!70!black}{\raisebox{0.3ex}{\scriptsize$\uparrow$\,#1}}}
\usepackage{algorithm}
\usepackage{algorithmic}

\usepackage[capitalize]{cleveref}

\usepackage{mdframed}
\mdfdefinestyle{mainthm}{
    linewidth=0.5pt,
    linecolor=black!40,
    backgroundcolor=black!3,
    innerleftmargin=10pt,
    innerrightmargin=10pt,
    innertopmargin=4pt,
    innerbottommargin=4pt,
    skipabove=6pt,
    skipbelow=6pt,
}

\theoremstyle{plain}
\newtheorem{theorem}{Theorem}[section]
\newtheorem{proposition}[theorem]{Proposition}

\theoremstyle{definition}

\theoremstyle{remark}
\newtheorem{remark}[theorem]{Remark}

\crefname{assumption}{Assumption}{Assumptions}
\Crefname{assumption}{Assumption}{Assumptions}

\newcommand{\bx}{\boldsymbol{x}}
\newcommand{\by}{\boldsymbol{y}}

\newcommand{\bv}{\boldsymbol{v}}
\newcommand{\bV}{\boldsymbol{V}}
\newcommand{\EE}{\mathbb{E}}
\newcommand{\RR}{\mathbb{R}}
\DeclareMathOperator{\Var}{Var}
\DeclareMathOperator{\Cov}{Cov}

\definecolor{colthetapop}{HTML}{017100}   
\definecolor{colTn}{HTML}{D64541}          
\definecolor{colTnabc}{HTML}{0077CC}       
\newcommand{\thetapop}{\textcolor{colthetapop}{\boldsymbol{T^*}}}
\newcommand{\Tn}{\textcolor{colTn}{\boldsymbol{T_n}}}
\newcommand{\Tone}{\textcolor{colTn}{\boldsymbol{T_1}}}
\newcommand{\Tnabc}{\textcolor{colTnabc}{\boldsymbol{T_n^{\mathrm{ABC}}}}}
\newcommand{\Tjack}{\textcolor{colTn}{\boldsymbol{T_n^{\mathrm{jack}}}}}
\newcommand{\Tloo}{\textcolor{colTn}{\boldsymbol{T_{n,-i}}}}
\newcommand{\Tboot}{\textcolor{colTn}{\boldsymbol{T_n^{*(b)}}}}
\newcommand{\Tbr}{\textcolor{colTn}{\boldsymbol{T_n^{\mathrm{boot}}}}}
\newcommand{\Tnd}{\textcolor{colTn}{\boldsymbol{T_{n,d}}}}

\newif\ificlr

\title{Analytical Correction for Subsampling Bias in Drifting Models}

\renewcommand{\thefootnote}{\fnsymbol{footnote}}
\author{Jiaru Zhang, Zeyun Deng, Juanwu Lu, Ziran Wang\footnotemark{} \setcounter{footnote}{0}, Ruqi Zhang\footnotemark{} \\
Purdue University \\
Correspondence to: \texttt{jiaru@purdue.edu}}

\begin{document}

\maketitle

\footnotetext[1]{Equal advising.}
\renewcommand{\thefootnote}{\arabic{footnote}}
\setcounter{footnote}{0}

\begin{abstract}
Drifting models are capable one-step generative models trained to follow a drifting field. The field combines attractive and repulsive softmax-weighted centroids over the data and current-generator distributions.
In practice, only a minibatch of $n$ samples from each distribution is available, and each centroid is approximated by an empirical estimate.
In this paper, we begin by showing that the minibatch centroid is in general a \emph{biased} estimator of the target centroid,
with a pointwise $O(1/n)$ bias arising from softmax self-normalization.
Correcting this bias requires the expectation over the full distribution, which is intractable.
We instead approximate the leading bias term from in-batch statistics and propose \emph{Analytical Bias Correction} (ABC), a closed-form plug-in adjustment.
We prove that ABC reduces the bias from $O(1/n)$ to $O(1/n^2)$, introduces no first-order increase in total variance,
and preserves convex-hull containment of the corrected centroid.
In practice, ABC requires only two additional lines of code and has negligible wall-time overhead under compiled execution.
Toy experiments confirm the theoretical $O(1/n)$ and $O(1/n^2)$ scaling.
On CIFAR-10, ABC reduces FID and trains faster, with the largest gains at small $n$, where the bias is most significant.
\end{abstract}

\section{Introduction}
\label{sec:intro}

Drifting models~\citep{deng2026drifting} are capable one-step generative models trained to follow a drifting field. Each generator sample defines a query point at which the field combines attraction toward a softmax-weighted centroid of the data distribution with repulsion from that of the current generator distribution. Despite being introduced only
recently, they have already been extended to molecular sampling, medical image synthesis, and                     
control~\citep{foffano2026rhc,hu2026boltzmann,lyu2026mri}, pointing to broad application potential.               
Unlike diffusion models or flow matching, whose target field is derived from a predefined forward process, drifting models construct this target field directly from data and generator distributions.                                                    
Since computing this field over the full distributions is prohibitive, the standard practice is to substitute a minibatch of $n$ samples from each distribution.

As already shown by \citet{deng2026drifting}, small $n$ degrades generation quality sharply.
What is less appreciated is \emph{why}.
In this paper, we start with an exact analysis of the $n{=}1$ case where the minibatch centroid reduces to a single sample, and find that its expectation reduces to the global mean rather than the softmax-weighted local average that defines the target.
For general $n$, we derive the explicit leading-order $O(1/n)$ bias of the softmax-weighted centroid at any query 
and show that the bias systematically \emph{contracts} the centroid toward a less localized average and away from the nearest data mode.
The bias is therefore not stochastic noise but a structural property of the minibatch estimator that should be \emph{corrected}.

\Cref{fig:illustration} visualizes this on a 2D toy example with three Gaussian clusters and a query (star) closer to the left cluster than the other two.
The target centroid (\textcolor{colthetapop}{\textbf{green diamond}}) is a softmax-weighted average pulled toward the nearest cluster.
A minibatch of $n{=}4$ samples (black rings) produces a biased centroid (\textcolor{colTn}{\textbf{red circle}}) shifted back toward the global mean of the three clusters,
with the residual pointing \emph{away} from the nearest cluster.
This is the centroid-contraction bias we characterize.

\begin{wrapfigure}[18]{r}{0.5\textwidth}
    \centering
    \ificlr
    \vspace{-0.4\baselineskip}
    \else
    \vspace{-1.4\baselineskip}
    \fi
    \includegraphics[width=0.5\textwidth]{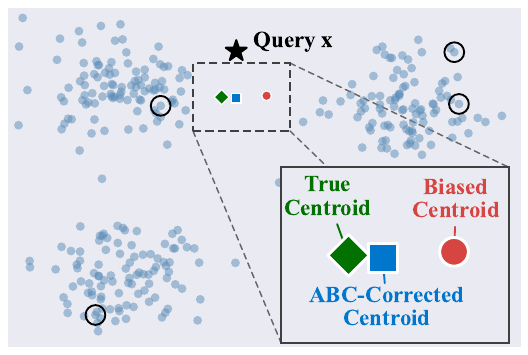}
    \caption{2D toy with three Gaussian clusters.
    From a minibatch of $n{=}4$ (black rings), the biased centroid (\textcolor{colTn}{\textbf{red circle}}) drifts away from the target centroid (\textcolor{colthetapop}{\textbf{green diamond}}), and
    ABC pulls it back (\textcolor{colTnabc}{\textbf{blue square}}).}
    \label{fig:illustration}
\end{wrapfigure}
Exact correction would require expectations over the full reference distributions, which are intractable.
We instead approximate the leading bias term from \emph{in-batch} statistics and derive \textbf{Analytical Bias Correction (ABC)}, a closed-form plug-in adjustment that reuses quantities already computed for the softmax.
Although ABC is itself estimated from in-batch statistics, we prove that it reduces the leading-order bias from $O(1/n)$ to $O(1/n^2)$, introduces no first-order increase in total variance, and preserves convex-hull containment of the corrected centroid.
\Cref{fig:illustration} shows the corrected centroid (\textcolor{colTnabc}{\textbf{blue square}}) sitting visibly closer to the target centroid than the biased one.
In practice, the correction adds two lines of code and has negligible wall-time overhead under compiled execution.
Toy experiments confirm the predicted $O(1/n)$ and $O(1/n^2)$ scaling directly,
and on CIFAR-10, ABC yields lower FID and reaches given FID targets with fewer training samples, with the largest gains at small~$n$ where the bias is most severe.
\ificlr
\else
The code is available at \url{https://github.com/lan-qing/DriftingABC}.
\fi

In summary, our contributions are:

\begin{enumerate}
    \item We show that, at any fixed query~$\bx$, the minibatch centroid in drifting models is a biased
    estimator of the target centroid, with a pointwise $O(1/n)$ bias arising from
    softmax self-normalization.

    \item We derive ABC by approximating the leading-order bias from in-batch
    statistics, and prove it achieves $O(1/n^2)$ residual bias, convex-hull
    containment, and no first-order increase in total variance.

    \item Toy experiments confirm the theoretical $O(1/n)$ and $O(1/n^2)$ scaling
    directly.  On CIFAR-10, ABC lowers FID and accelerates convergence, with the largest gains at small~$n$.
\end{enumerate}

\subsection{Related Work}
\label{sec:related}

\paragraph{Drifting models.}
Drifting models were introduced by \citet{deng2026drifting} as one-step generators trained with kernel-based attraction-repulsion fields.
Since then, a number of concurrent works have developed theoretical views and alternative formulations~\citep{cao2026gradient,franz2026conservative,he2026sinkhorn,lai2026unified},
and extended drifting models beyond image generation to molecular sampling, medical image synthesis, and control~\citep{foffano2026rhc,hu2026boltzmann,lyu2026mri}.
Among them, the concurrent work of~\citet{lai2026unified} establishes an equivalence between drifting and score-matching under Gaussian/Laplace kernels, and in the course of this analysis, they also remark that finite-sample mini-batch drift estimation with normalized kernel weights forms a ratio-type Monte Carlo estimator and may therefore exhibit variance and bias.
However, their work focuses on the population-level connection to score matching, not on the finite-sample estimator itself.
Complementary to that work, we make the bias precise by characterizing it as a generic $O(1/n)$ centroid contraction,
and derive an analytical correction with $O(1/n^2)$ residual bias.

\paragraph{Self-normalization, minibatching, and generative transport.}
Bias in ratio estimators and self-normalized Monte Carlo estimators has long been understood.
Classical analyses characterize the leading $O(1/n)$ term,
and bias-reduction schemes such as the jackknife or Bias Reduced Self-Normalized Importance Sampling (BR-SNIS) can push the residual to $O(1/n^2)$~\citep{beale1962,cardoso2022,cochran1977,owen2013,quenouille1956}.
However, these schemes typically incur extra kernel launches and $O(nD)$ leave-one-out bookkeeping per evaluation or require auxiliary samples,
which is prohibitive in drifting models where the centroid must be recomputed at every training step.
The jackknife is additionally unstable when one sample dominates the softmax, because removing that sample produces a leave-one-out centroid far from the rest and the jackknife's bias-correction step amplifies this outlier.
Our contribution is a closed-form correction that integrates naturally with the existing softmax computation.
Related effects also arise when machine learning objectives approximate a full normalizer with sampled or truncated support, such as sampled softmax and efficient attention~\citep{bengio2008,zheng2022}.
Our setting is different from both lines of work and from flow matching~\citep{lipman2023},
as drifting trains against a softmax-weighted centroid over the \emph{entire} reference distribution, so minibatching produces a self-normalized target estimator rather than an unbiased sample average.

\section{Preliminaries}
\label{sec:preliminaries}

\textbf{The drifting model.}
The drifting model~\citep{deng2026drifting} generates samples by training-time evolution of a pushforward distribution.
At each step, a \emph{drifting field} $\bV_{p,q}(\bx)$ prescribes the target direction at each generated sample $\bx$ for training the model, depending on both the data distribution~$p$ and the current generator distribution~$q$.
The field is defined as:
$\bV_{p,q}(\bx) = \bV_p^+(\bx) - \bV_q^-(\bx)$,
where the \emph{positive} (attraction) and \emph{negative} (repulsion) components are softmax-weighted mean-shift vectors:
\begin{equation}\label{eq:meanshift}
  \bV_p^+(\bx)
  := \frac{1}{Z_p(\bx)}\,\EE_{\by^+ \sim p}\!\bigl[k(\bx,\by^+)(\by^+ - \bx)\bigr],
  \qquad
  \bV_q^-(\bx)
  := \frac{1}{Z_q(\bx)}\,\EE_{\by^- \sim q}\!\bigl[k(\bx,\by^-)(\by^- - \bx)\bigr],
\end{equation}
with normalization factors $Z_p(\bx) = \EE_p[k(\bx,\by^+)]$ and $Z_q(\bx) = \EE_q[k(\bx,\by^-)]$.
The kernel $k(\cdot,\cdot)$ can be any positive similarity function, and we adopt $k(\bx,\by) = \exp(-\|\bx-\by\|_2/\tau)$ by default with temperature~$\tau$ following~\citet{deng2026drifting}.
Intuitively, $\bV_p^+$ attracts $\bx$ toward nearby data samples and $\bV_q^-$ repels it from other generated samples, and
the combined field drives the generator distribution toward the data distribution.

\textbf{Minibatch approximation.}
In practice, the expectations over $p$ and $q$ in~\Cref{eq:meanshift} are intractable.
Instead, at each training step, $n$ reference samples $\by_1,\dots,\by_n$ are drawn as a minibatch,
and the positive field $\bV_p^+(\bx)$ is replaced by
\begin{equation}\label{eq:minibatch}
  \sum_{i=1}^{n} \alpha_i\,(\by_i - \bx),
  \qquad \text{where} \quad
  \alpha_i = \frac{k(\bx,\by_i)}{\sum_{j=1}^{n} k(\bx,\by_j)},
\end{equation}
which is a softmax-weighted combination over the minibatch.
The negative field is approximated similarly from other concurrently generated samples, excluding the query itself.

\section{The Bias Problem}
\label{sec:bias}
In this section, we first show that the minibatch estimator is biased for $n{=}1$ with an exact result.
As the exact bias admits no closed form for general~$n$, we then analyze the zero-bias conditions and quantify the leading-order bias.
Note that the positive component is a softmax-weighted mean-shift over the data distribution $p$ and the negative is its structural counterpart over the generator distribution $q$, and both have the same ratio-of-averages form. 
For simplicity, we focus on the positive side in what follows and omit the $p$ subscript on $\EE[\cdot]$. Similar statements apply to the negative side by replacing $p$ with $q$ throughout.
For ease of reference, we also present a table of notations in \Cref{app:notation}.

\textbf{Centroid reformulation.}
To simplify the analysis, we first rewrite the formulation in centroid form.
Concretely, we rewrite the mean-shift vector in~\Cref{eq:meanshift} as $\bV(\bx) = \thetapop(\bx) - \bx$, where
\begin{equation}\label{eq:centroid_pop}
  \thetapop(\bx)
  \;:=\; \frac{\EE\!\bigl[k(\bx,\by)\,\by\bigr]}
              {\EE\!\bigl[k(\bx,\by)\bigr]}
\end{equation}
is the \emph{target centroid}, i.e., the softmax-weighted average of the reference distribution.
The minibatch approximation in~\Cref{eq:minibatch} estimates $\thetapop$ by
\begin{equation}\label{eq:centroid}
  \Tn(\bx) \;:=\; \sum_{i=1}^n \alpha_i\,\by_i,
  \qquad \text{where} \quad
  \alpha_i = \frac{k(\bx,\by_i)}{\sum_{j=1}^n k(\bx,\by_j)}.
\end{equation}
Since $\bx$ is fixed, it suffices to analyze the bias of $\Tn$ relative to $\thetapop$.

\textbf{The bias in the single-sample case $n=1$.}
When $n{=}1$, a single sample is drawn from the distribution and used directly as the centroid estimate, i.e., $\alpha_1=1$ and $\Tone=\by$, and the softmax-weighted averaging is lost entirely.
The expected estimate therefore reduces to the \emph{global mean} $\EE[\by]$,
whereas the target $\thetapop$ is a \emph{softmax-weighted} average that puts more mass on points close to~$\bx$.
The bias is exactly the gap between the two.
To state it precisely, write $w = k(\bx,\by)$ for the kernel weight, and we have:
\begin{theorem}[Bias for $n{=}1$]\label{thm:bias_n1}
For a single sample $\by\sim p$ with $w=k(\bx,\by)\geq 0$, $\Pr(w>0)=1$ (so $\Tone$ is well defined), $\EE[w]\in(0,\infty)$, $\EE[\|\by\|]<\infty$, and $\EE[\|w\by\|]<\infty$:
\begin{equation}\label{eq:bias_n1}
  \EE[\Tone] - \thetapop
  \;=\; -\,\frac{\Cov(w,\,\by)}{\EE[w]},
\end{equation}
where $\Cov(w,\by) := \EE[(w-\EE w)(\by-\EE\by)] \in \RR^D$ is the componentwise covariance between the scalar $w$ and the vector $\by$.
\end{theorem}
\begin{proof}
For $n{=}1$ the centroid reduces to the sample itself, so $\EE[\Tone]=\EE[\by]$.
The target centroid is $\thetapop=\EE[w\by]/\EE[w]$.
Their difference is
\[
  \EE[\by] - \frac{\EE[w\by]}{\EE[w]}
  \;=\; \frac{\EE[w]\,\EE[\by]-\EE[w\by]}{\EE[w]}
  \;=\; -\,\frac{\Cov(w,\by)}{\EE[w]}.  \qedhere
\]
\end{proof}
\noindent
\ificlr
The bias is governed by $\Cov(w,\by)$, the covariance between kernel weight and sample position.
Since points closer to~$\bx$ have larger~$w$,
this covariance is generically nonzero, 
and the negative sign causes the estimator's expectation to shift away from the softmax-weighted target toward a flatter average.
\else
The bias is governed by $\Cov(w,\by)$, the covariance between kernel weight and sample position.
Since points closer to~$\bx$ receive larger~$w$,
this covariance is generically nonzero, 
and the negative sign in \Cref{eq:bias_n1} causes the estimator's expectation to shift away from the softmax-weighted target toward a flatter average.
\fi

\textbf{Conditions for the bias to vanish in general $n$.}
The $n{=}1$ result admits an exact formula because the softmax weight trivially equals one.
For general $n$, the weights $\alpha_i = k(\bx,\by_i)/\sum_j k(\bx,\by_j)$ couple all samples through a shared denominator,
and $\EE[\Tn]$ no longer has a closed-form expression.
Nonetheless, we can identify two natural sufficient conditions under which the bias vanishes:
\begin{theorem}[Zero-bias conditions]\label{thm:zero_bias}
Let $\by_1,\dots,\by_n\overset{\mathrm{i.i.d.}}{\sim}p$ with $w=k(\bx,\by)\geq 0$, $\EE[w]\in(0,\infty)$, $\EE[\|\by\|]<\infty$, $\EE[\|w\by\|]<\infty$, and $\Pr(w>0)=1$ (so $\sum_i w_i>0$ a.s.\ and $\Tn$ is well defined).
  $\EE[\Tn]=\thetapop$ if either:
  \begin{enumerate}
  \item[\rm(i)] \emph{Constant weight}: $w=k(\bx,\by)$ is constant
    a.s.\ (e.g., all points equidistant from $\bx$ for a radial
    kernel); or
  \item[\rm(ii)] \emph{Kernel-symmetric data}: $p$ is symmetric about
    $\bx$ and the kernel is reflection-invariant
    ($k(\bx,\by)=k(\bx,2\bx-\by)$; holds for any radial kernel).
  \end{enumerate}
\end{theorem}

\begin{proof}
\emph{(i).}
Constant $w$ gives $\alpha_i=1/n$, so $\Tn=n^{-1}\sum_i\by_i$ and $\EE[\Tn]=\EE[\by]=\thetapop$.

\emph{(ii).}
First, we show $\thetapop=\bx$.
By symmetry, $\EE[w(\by-\bx)] = -\EE[w(\by-\bx)]$ (the reflection $\by\mapsto 2\bx-\by$ preserves both $p$ and $w$),
so $\EE[w(\by-\bx)]=\mathbf{0}$ and $\thetapop=\EE[w\by]/\EE[w]=\bx$.
For $\Tn$, reflecting the minibatch does not change its probability or the weights, but the centroid becomes $2\bx-\Tn$.
Therefore $\EE[\Tn]=\EE[2\bx-\Tn]=2\bx-\EE[\Tn]$, giving $\EE[\Tn]=\bx=\thetapop$.
\end{proof}

\noindent
Both conditions are restrictive.
Condition (i) requires all points to be equidistant from the query,
and (ii) requires the data to be symmetric about it.
Outside these special cases, how large is the bias?
The next result answers this for general~$n$.

\textbf{The magnitude of the bias in general $n$.}
Although an exact closed form is unavailable, we can still derive the leading order of the bias, stated in the following theorem.
\begin{mdframed}[style=mainthm]
\begin{theorem}[Leading-order bias]\label{thm:ratio_bias}
  Let $\by_1,\dots,\by_n\overset{\mathrm{i.i.d.}}{\sim}p$ with
  $w=k(\bx,\by)\in[0,M]$ and $\|\by\|\leq C$ almost surely for some constants $M,C>0$, and $\Pr(w>0)=1$ (so $\sum_i w_i>0$ a.s.\ and $\Tn$ is well defined).
  For fixed $\bx$ with $\EE[w]>0$, the bias of the minibatch centroid satisfies
\begin{equation}\label{eq:ratio_bias}
  \EE[\Tn] - \thetapop
  \;=\; -\,\frac{1}{n}\,
    \frac{\EE\!\bigl[w^2(\by-\thetapop)\bigr]}{(\EE[w])^2}
  \;+\; O(n^{-2})
\end{equation}
as $n\to\infty$.
\end{theorem}
\end{mdframed}
The proof views $\Tn$ as a ratio of two sample averages, $\Tn = \sum_i w_i\by_i \big/ \sum_i w_i$. Each sum is individually unbiased, but the expectation of the ratio differs from the ratio of expectations, and the leading-order gap is controlled by the classical delta method for self-normalized estimators~\citep{hesterberg1995,owen2013,tin1965}.
See \Cref{app:delta} for the full derivation.

The term $\EE[w^2(\by-\thetapop)]$ is a \emph{$w^2$-weighted displacement}.
Since $w^2$ amplifies large weights more aggressively than $w$ does, this vector points from $\thetapop$ toward the more localized $w^2$-weighted center of mass.
The \textbf{minus sign} reverses it, so the minibatch estimator is systematically pulled \emph{away} from this local concentration toward a flatter average, a \textbf{centroid contraction}.
In clustered reference distributions such as the toy in \Cref{fig:illustration}, this is visible as a pull from the nearest mode toward the global mean.
The $1/n$ factor means the bias shrinks with larger minibatch size but remains nonzero at finite~$n$ in general cases.

\begin{remark}[Assumptions in practice]
\label{rem:assumptions}
The boundedness and moment conditions in \Cref{thm:bias_n1,thm:zero_bias,thm:ratio_bias}, which are inherited by the ABC results in \Cref{sec:abc}, hold naturally in the drifting-model setting for both the positive (data) and negative (current generator) sides.
The exponential kernel is strictly positive and bounded, $k(\bx,\by) = \exp(-\|\bx-\by\|_2/\tau) \in (0,1]$, so $\Pr(w>0)=1$.
The reference features $\by$ come from a frozen encoder applied to a finite dataset on the positive side and to a batch of generated images on the negative side. In both cases, $\|\by\|$ is bounded almost surely, giving finite moments of all orders.
\end{remark}

\textbf{From diagnosis to correction.}
Given that the bias is $O(1/n)$ with a \emph{closed-form} leading term, a natural plan is to subtract it.
Two questions arise immediately.
First, the leading term involves full-distribution expectations that are not available at training time, so can we plug in minibatch estimates without reintroducing an $O(1/n)$ error?
Second, bias reduction usually comes at the cost of inflated variance, and can we avoid paying that cost?
The next section answers both, along with a geometric sanity check on the corrected estimator.

\section{Analytical Bias Correction}
\label{sec:abc}

Having established that the minibatch centroid carries an $O(1/n)$ bias,
we now construct a closed-form correction that can be computed from the \emph{same minibatch} at negligible cost.
We derive the correction and prove that it reduces the bias to $O(1/n^2)$, introduces no first-order increase in total variance, and preserves convex-hull containment. We briefly note its relationship with temperature tuning.

\subsection{Construction}
\label{sec:abc_construction}

\Cref{thm:ratio_bias} gives the leading-order bias $-\frac{1}{n}\EE[w^2(\by-\thetapop)]/(\EE[w])^2$,
but this expression involves full-distribution expectations that are unavailable at training time.
To solve this problem, we use the key observation that all quantities can be estimated from the same minibatch used to compute $\Tn$.
Specifically, we replace $\EE[\cdot]$ with $\frac{1}{n}\sum_i$, $\thetapop$ with $\Tn$, and $\EE[w]$ with $\frac{1}{n}\sum_i w_i$.
Substituting these into the bias formula and simplifying yields the plug-in correction:
\begin{align}\label{eq:delta}
  \Delta_n
  &\;:=\;
    -\frac{1}{n}\,
    \frac{\tfrac{1}{n}\sum_i w_i^2\,( \by_i - \Tn)}
         {\bigl(\tfrac{1}{n}\sum_i w_i\bigr)^{\!2}}
  \;=\;
    \frac{\sum_i w_i^2\,(\Tn - \by_i)}
         {{\bigl(\sum_i w_i\bigr)}^{2}}
  \;=\;
    \sum_{i=1}^n \alpha_i^2\,(\Tn - \by_i),
\end{align}
where $\alpha_i = w_i / \sum_j w_j$ are the softmax weights as defined in \Cref{eq:centroid}.
This yields our Analytical Bias Correction (ABC) estimator, which subtracts this correction from the original centroid:
\begin{equation}\label{eq:abc}
  \boxed{\;
  \Tnabc
  \;:=\; \Tn - \Delta_n
  \;=\; \Bigl(1-\textstyle\sum_i\alpha_i^2\Bigr)\,\Tn
        \;+\; \sum_{i=1}^n \alpha_i^2\,\by_i.
  \;}
\end{equation}
As a sanity check, when weights are uniform ($\alpha_i = 1/n$), $\sum_i\alpha_i^2 = 1/n$ and $\Delta_n = \mathbf{0}$.
This is consistent with \Cref{thm:zero_bias}(i), which shows that constant weights yield an unbiased estimator.
We would like to note that the plug-in correction in \Cref{eq:abc} can also be viewed as a specialization of the classical first-order bias correction for ratio estimators~\citep{tin1965}.

\subsection{Formal Guarantees}
\label{sec:abc_guarantees}

Since the correction $\Delta_n$ is itself estimated from the same biased minibatch,
it is natural to ask whether it truly cancels the bias rather than introducing new errors.
The following theorem confirms that ABC reduces the bias by one order of~$n$.

\begin{mdframed}[style=mainthm]
\begin{theorem}[Bias reduction]\label{thm:abc}
Under the assumptions of \Cref{thm:ratio_bias},
the ABC estimator reduces the bias from $O(n^{-1})$ to $O(n^{-2})$:
\begin{equation}\label{eq:abc_bias}
  \EE\bigl[\Tnabc - \thetapop\bigr] = O(n^{-2}).
\end{equation}
\end{theorem}
\end{mdframed}
\begin{proof}[Proof sketch]
Decompose $\Delta_n$ by inserting $\thetapop$:
\[
  \Delta_n
  \;=\; \underbrace{\sum_i \alpha_i^2\,(\thetapop-\by_i)}_{\text{(A)}}
  \;+\; \underbrace{\Bigl(\sum_i\alpha_i^2\Bigr)
        (\Tn-\thetapop)}_{\text{(B)}}.
\]
Term~(A) is the sample analog of the leading bias in \Cref{thm:ratio_bias} and matches it up to $O(n^{-2})$ after a Taylor expansion of the denominator around $\EE[w]$.
Term~(B) has expectation $O(n^{-2})$ because $\EE[\sum_i\alpha_i^2] = O(n^{-1})$ from bounded weights and $\EE[\Tn-\thetapop] = O(n^{-1})$ from \Cref{thm:ratio_bias}, with the residual covariance also $O(n^{-2})$ by Cauchy--Schwarz.
Combining, $\EE[\Delta_n] = (\EE[\Tn]-\thetapop) + O(n^{-2})$, hence $\EE[\Tnabc-\thetapop] = O(n^{-2})$.
See \Cref{app:abc_bias} for the full proof.
\end{proof}

Beyond bias reduction, two practical concerns remain.
First, since the correction subtracts $\Delta_n$ from $\Tn$, the corrected centroid could in principle extrapolate beyond the minibatch points.
Second, reducing bias often increases variance. 
We address both below.

\begin{proposition}[Convex hull guarantee]\label{prop:convex}
Under the conditions of \Cref{thm:abc},
$\Tnabc = \sum_i \gamma_i\,\by_i$ with
$\gamma_i = \alpha_i(1-\sum_j\alpha_j^2)+\alpha_i^2\geq 0$ and
$\sum_i\gamma_i=1$.
Hence $\Tnabc$ lies in the convex hull of
$\{\by_1,\dots,\by_n\}$.
\end{proposition}

\begin{proof}
$\sum_j\alpha_j^2\leq\sum_j\alpha_j=1$, so $1-\sum_j\alpha_j^2\geq 0$ and $\gamma_i\geq 0$.
Summing: $\sum_i\gamma_i=(1-\sum_j\alpha_j^2)+\sum_j\alpha_j^2=1$.
\end{proof}

\noindent
That is, $\Tnabc$ is always a weighted average of the minibatch points with non-negative weights,
and it never extrapolates beyond the data.

Intuitively, $\Delta_n$ is an $O(1/n)$ correction while the fluctuations of $\Tn$ are of order $O(1/\sqrt{n})$, so subtracting a term that is smaller than the existing noise by a factor of $\sqrt{n}$ should not change the variance at leading order. The next proposition makes this precise.

\begin{proposition}[No first-order total-variance increase]\label{prop:variance}
Under the conditions of \Cref{thm:abc}:
\begin{equation}\label{eq:var_abc}
  \EE\bigl[\|\Tnabc-\EE[\Tnabc]\|^2\bigr]
  \;=\; \EE\bigl[\|\Tn-\EE[\Tn]\|^2\bigr] + O(n^{-3/2})
  \;=\; O(n^{-1}).
\end{equation}
\end{proposition}
\begin{proof}[Proof sketch]
The proof decomposes $\Tnabc = \Tn - \Delta_n$ and tracks each piece's contribution to the variance.
First, $\Tn$ is a smooth ratio of two i.i.d.\ sample means, so the delta method gives variance $O(n^{-1})$.
Second, since $\Tn$ is a convex combination of the bounded points, $\|\by_i-\Tn\|\leq 2C$, giving $\|\Delta_n\|^2\leq 4C^2(\sum_i\alpha_i^2)^2$. The same bounded-weight argument as in \Cref{thm:abc}, $\sum_i\alpha_i^2\leq M/\sum_j w_j$, yields $\EE[(\sum_i\alpha_i^2)^2]=O(n^{-2})$ and hence $\EE\|\Delta_n\|^2=O(n^{-2})$.
The cross term in the variance of $\Tn - \Delta_n$ is then $O(n^{-3/2})$ by Cauchy--Schwarz, giving the claimed residual.
See \Cref{app:variance} for the full proof.
\end{proof}
That is, ABC preserves the same leading-order $O(1/n)$ variance as the uncorrected estimator.

\subsection{Implementation}
\label{sec:abc_impl}

ABC is a \emph{drop-in replacement} for the minibatch centroid in \Cref{eq:centroid} that leaves the rest of the drifting-model training pipeline untouched, requiring no auxiliary network and no extra minibatch.
Since the drifting field contains both a positive and a negative centroid, each with $O(1/n)$ bias, ABC is applied independently to both.
Concretely, given the softmax weights $\texttt{alpha}\in\RR^{B\times n}$ and features $\texttt{Y}\in\RR^{n\times D}$, where $B$ is the number of queries processed in parallel and $D$ is the feature dimension, the corrected centroid requires only two additional lines in PyTorch:
\begin{lstlisting}
a2 = alpha ** 2   
T_abc = (1 - a2.sum(-1, keepdim=True)) * (alpha @ Y) + a2 @ Y
\end{lstlisting}
The full drifting-model training loop with ABC is summarized in \Cref{alg:abc_training} of \Cref{app:algorithm}.

\paragraph{Complexity analysis.}
We report costs per query.
The standard centroid $\Tn = \texttt{alpha @ Y}$ costs $O(nD)$.
ABC adds an element-wise square of the weights ($O(n)$),
a reduction for $\sum_i\alpha_i^2$ ($O(n)$),
and one additional matrix multiply $\texttt{a2 @ Y}$ ($O(nD)$).
Therefore, the total overhead is $O(nD)$, which is the same order as the baseline.

\paragraph{On the kernel temperature.}
Since both subsampling bias and a high temperature push the centroid toward the global mean, a natural question is whether tuning the temperature $\tau$ alone could substitute for ABC.
A first-order analysis in \Cref{app:temperature} shows it cannot, as
the bias depends on the \emph{local} weight distribution at each query, whereas $\tau$ is a single global scalar, and lowering $\tau$ also reduces the effective sample size.
A side effect is that the $\tau$ inherited from the standard estimator implicitly absorbs part of its $O(1/n)$ bias, so once ABC removes that bias, the optimal $\tau$ for the corrected estimator may shift.


\section{Experiments}
\label{sec:experiments}

\subsection{Toy Experiments}
\label{sec:bias_validation}
\ificlr
\begin{wrapfigure}[15]{r}{0.55\textwidth}
\else
\begin{wrapfigure}[17]{r}{0.58\textwidth}
\fi
  \centering
  \begin{subfigure}[b]{0.48\linewidth}
      \includegraphics[width=\textwidth]{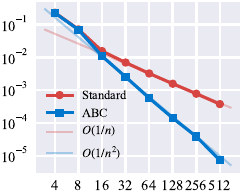}
      \caption{$\tau = 0.1$}
  \end{subfigure}
  \begin{subfigure}[b]{0.48\linewidth}
      \includegraphics[width=\textwidth]{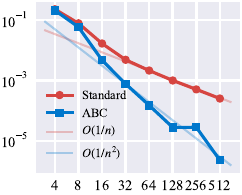}
      \caption{$\tau = 0.2$}
  \end{subfigure}
  \caption{Bias norm $\|\EE[\Tn]-\thetapop\|$ vs.\ $n$ on a 2D four-mode Gaussian toy, log-log axes.
  Left: tight kernel ($\tau{=}0.1$); right: wider kernel ($\tau{=}0.2$).
  Standard traces slope $-1$ ($O(1/n)$), ABC slope $-2$ ($O(1/n^2)$).}
  \label{fig:bias_loglog}
\end{wrapfigure}

We first verify the bias scaling predicted by \cref{thm:ratio_bias} on a controlled toy problem before turning to real generative modeling.
We construct a four-mode isotropic Gaussian in $\RR^{2}$ and draw a reference pool of $N = 10^{6}$ points.
The target centroid $\thetapop$ is computed as the softmax-weighted centroid using the full pool,
which at $N{=}10^6$ closely approximates the true expectation.
For each sample size $n$,
we subsample $n$ points from the pool, compute $\Tn$, and repeat $200{,}000$ times to estimate the bias norm $\|\EE[\Tn] - \thetapop\|$ at each query, then average across $200$ queries to summarize the pointwise scaling.
Setup details are in \Cref{app:toy}.

In \Cref{fig:bias_loglog}, the Standard centroid traces a slope-$-1$ line, and the ABC curve traces slope $-2$ over roughly two decades of $n$.
At small $n$, e.g., $n{=}4, 8$, the two curves nearly overlap, because higher-order residuals in the ABC expansion are comparable in magnitude to the $O(1/n)$ bias of the Standard estimator when $n$ is small.
Even within this range, the ABC bias never exceeds the Standard bias at any $n$ we tested.
For larger $n$, the observed slopes closely match $-1$ and $-2$, empirically validating \Cref{thm:ratio_bias,thm:abc}.

We also compare ABC with three classical bias-reduction baselines, namely the jackknife, a bootstrap correction, and the BR-SNIS estimator of \citet{cardoso2022}, on a lighter CPU variant of this toy setup.
Jackknife, Bootstrap, and ABC all operate on a fixed $n$-sample pool. Among them, ABC is the fastest correction while attaining comparable bias.
BR-SNIS attains the smallest bias but requires roughly $Kn$ fresh reference samples per estimate (with $K=10$), a substantially larger sample budget than the others.
Together, these results support ABC as the most efficient correction under the fixed-$n$ sample budget typical of drifting-model training.
The details are in \Cref{app:baselines}.

\subsection{Image Generation Setup}
\label{sec:setup}

We evaluate ABC on the DriftDiT-Small architecture \citep{deng2026drifting},
a 26.6M-parameter diffusion transformer that uses a frozen DINOv2~\citep{oquab2024dinov2} multi-scale encoder and bidirectional kernel normalization.
Main experiments use CIFAR-10~\citep{krizhevsky2009learning}.
We sweep the per-class sample count $n \in \{8, 16, 32, 64\}$ used for both positive and negative centroids, where $n{=}64$ is the default setting of the original paper \citep{deng2026drifting}.
We train two variants that differ only in how the centroid is computed, the standard (uncorrected) estimator and the ABC-corrected one from \cref{sec:abc}. 

\paragraph{Training.}
We follow the settings of \citet{deng2026drifting}, including the temperatures $\tau\in\{0.02, 0.05, 0.2\}$.
The comparison is therefore mildly favorable to Standard, since the temperatures were chosen by \citet{deng2026drifting} for the standard estimator and are not retuned for ABC.
To equalize total positive samples seen across configurations, we scale training duration inversely with $n$ so that $n{=}64$ trains for $500$ epochs, $n{=}32$ for $1\,000$, etc.
Full hyperparameters and computation details are in \Cref{app:hyperparams}.

\paragraph{Evaluation.}
We report Fr\'{e}chet Inception Distance (FID)~\citep{heusel2017gans} computed on 50\,k generated samples, using the same sampling configuration for Standard and ABC.
All reported quantities are mean~$\pm$ max-deviation across three random seeds unless noted otherwise.

\begin{figure*}[t]
  \centering
  \begin{subfigure}[b]{0.24\textwidth}
      \includegraphics[width=\textwidth]{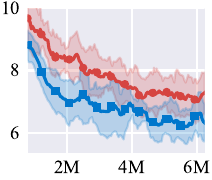}
      \caption{$n{=}8$}
  \end{subfigure}\hfill
  \begin{subfigure}[b]{0.24\textwidth}
      \includegraphics[width=\textwidth]{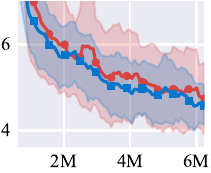}
      \caption{$n{=}16$}
  \end{subfigure}\hfill
  \begin{subfigure}[b]{0.24\textwidth}
      \includegraphics[width=\textwidth]{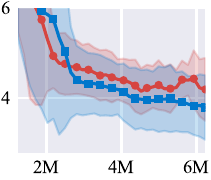}
      \caption{$n{=}32$}
  \end{subfigure}\hfill
  \begin{subfigure}[b]{0.24\textwidth}
      \includegraphics[width=\textwidth]{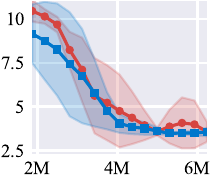}
      \caption{$n{=}64$}
  \end{subfigure}
  \caption{FID vs.\ total positive samples seen (smoothed with a rolling average).
  Each panel fixes a per-class sample count $n$ and compares the Standard variant of the drifting model (\textcolor{colTn}{\textbf{red circles}}) with its ABC-corrected counterpart (\textcolor{colTnabc}{\textbf{blue squares}}).
  Across all four values of~$n$, the ABC curve lies below the baseline through most of training.}
  \label{fig:convergence}
\end{figure*}

\subsection{Main Results}
\label{sec:main_results}

\Cref{fig:convergence} shows the FID trajectories throughout training for all four sample sizes,
and \Cref{tab:main_results} summarizes both the last-5-evaluation average and the best snapshot FID.
ABC improves the last-5-average FID at every tested $n$, and the best FID in most cases.
The largest gains occur at the smallest $n$, consistent with the $O(1/n)$ bias scaling of \Cref{thm:ratio_bias}.
This advantage is also visible in the smoothed trajectories of \Cref{fig:convergence}.
These observations match the picture from \Cref{sec:bias}. They are consistent with the $O(1/n)$ pointwise bias having a measurable impact on generation quality, which ABC mitigates in the small-$n$ regime.
Single-seed results on CIFAR-100 presented in \Cref{app:cifar100} also show that ABC improves both last-5 average and best FID, supporting consistency with the CIFAR-10 trends.

To complement FID, we additionally report the improved precision
and recall of \citet{kynkaanniemi2019improved} in \Cref{app:prdc}.
ABC improves recall rather than precision, consistent with the theoretical analysis of \Cref{thm:ratio_bias}.

\begin{table}[t]
\centering
\caption{FID on CIFAR-10 for Standard vs ABC across per-class sample count $n$.}
\label{tab:main_results}
\resizebox{\linewidth}{!}{%
\begin{tabular}{@{}ccccccc@{}}
\toprule
\multirow{2}{*}{$n$} & \multirow{2}{*}{\shortstack{Total positives\\per class}} & \multirow{2}{*}{Epochs} & \multicolumn{2}{c}{Last-5 average FID ($\downarrow$)} & \multicolumn{2}{c}{Best FID ($\downarrow$)} \\
\cmidrule(lr){4-5} \cmidrule(lr){6-7}
 &  &  & Standard & ABC & Standard & ABC \\
\midrule
  8 & $6.24$\,M & 4\,000 & $7.31 \pm 0.94$ & $\mathbf{6.40 \pm 0.36}$\,\dnred{0.91} & $6.83 \pm 0.95$ & $\mathbf{5.81 \pm 0.34}$\,\dnred{1.02} \\
 16 & $6.24$\,M & 2\,000 & $4.79 \pm 1.04$ & $\mathbf{4.66 \pm 0.66}$\,\dnred{0.13} & $4.46 \pm 1.00$ & $\mathbf{4.30 \pm 0.50}$\,\dnred{0.17} \\
 32 & $6.24$\,M & 1\,000 & $4.34 \pm 0.60$ & $\mathbf{3.89 \pm 0.73}$\,\dnred{0.45} & $4.02 \pm 0.50$ & $\mathbf{3.71 \pm 0.69}$\,\dnred{0.31} \\
 64 & $6.24$\,M &   500 & $3.77 \pm 0.89$ & $\mathbf{3.52 \pm 0.14}$\,\dnred{0.24} & $\mathbf{3.24 \pm 0.03}$ & $3.40 \pm 0.08$\,\upred{0.16} \\
\bottomrule
\end{tabular}%
\ificlr
\vspace{-3em}
\fi
}
\end{table}

\subsection{Training Speed}
\label{sec:training_speed}

Beyond the endpoint FID, we ask how quickly each estimator reaches a target.
Following the literature \citep{sauer2021projected,karras2020training}, we record the first epoch at which FID drops below the threshold and convert to positive samples seen.
As the $6.24$\,M-per-class budget is fixed across $n$, samples-to-FID is comparable across batch sizes.
\Cref{tab:samples_to_fid} reports an attainment rate ``$k/3$'' and a sample-count summary over the seeds that reach the threshold.
Despite the large variance inherent to first-hit-time measurements on noisy FID curves, ABC reaches the target with comparable or fewer samples at every $n$, with the strongest effect at small~$n$.
For example, at $n{=}8$ with FID~$\leq 10$, ABC needs about $40\%$ of Standard's samples.
With the tighter FID~$\leq 7$ threshold, only ABC reliably attains it across seeds ($3/3$ vs.\ $1/3$).
The largest gains occur at the smallest $n$, consistent with the $O(1/n)$ bias shrinking with~$n$.

\begin{table}[t]
\centering
\caption{Training samples (millions of positives, summed over all classes) required to first reach FID~$\leq 10$ and FID~$\leq 7$.
$(k/3)$ flags attainment rate when fewer than three seeds reach it.
}
\label{tab:samples_to_fid}
\begin{tabular}{@{}ccccc@{}}
\toprule
\multirow{2}{*}{$n$} & \multicolumn{2}{c}{Required Samples to FID $\leq 10$ ($\downarrow$)} & \multicolumn{2}{c}{Required Samples to FID $\leq 7$ ($\downarrow$)} \\
\cmidrule(lr){2-3}\cmidrule(lr){4-5}
 & Standard & ABC & Standard & ABC \\
\midrule
$8$  & $8.19\pm4.29$ & $\mathbf{3.38\pm0.52}$\,\dnred{4.81} & $26.52$ (1/3) & $\mathbf{21.84\pm12.48}$ (3/3)\,\dnred{4.68} \\
$16$ & $4.94\pm1.04$ & $\mathbf{4.42\pm0.52}$\,\dnred{0.52} & $11.44\pm7.28$ & $\mathbf{7.80\pm4.68}$\,\dnred{3.64} \\
$32$ & $8.84\pm3.64$ & $8.84\pm3.64$ & $13.52\pm5.20$ & $\mathbf{10.92\pm3.12}$\,\dnred{2.60} \\
$64$ & $16.64\pm10.40$ & $\mathbf{13.52\pm8.32}$\,\dnred{3.12} & $30.16\pm4.16$ & $\mathbf{27.04\pm11.44}$\,\dnred{3.12} \\
\bottomrule
\ificlr
\vspace{-2em}
\fi
\end{tabular}
\end{table}

\subsection{Ablation Study on Positive vs.\ Negative Correction}
\label{sec:ablation}

\ificlr
\begin{wrapfigure}[15]{r}{0.42\textwidth}
    \centering
    \vspace{-\baselineskip}
\else
\begin{wrapfigure}[17]{r}{0.48\textwidth}
    \centering
\fi
    \includegraphics[width=\linewidth]{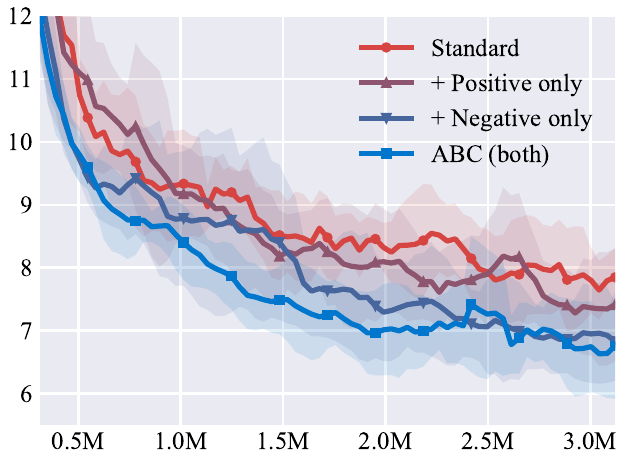}
    \caption{FID trajectories at $n{=}8$ with four variants. Bands show mean $\pm$ max-deviation across three seeds.}
    \label{fig:ablation_trajectory}
    \end{wrapfigure}
ABC corrects both the positive centroid over real data and the negative centroid over generated samples.
To isolate each contribution, we run an ablation at $n{=}8$ with four variants and report FID trajectories smoothed across seeds in \Cref{fig:ablation_trajectory}.
First, every correction variant lies below the Standard curve, confirming that both sides of the centroid contribute bias that training can exploit.
Second, the negative-only variant outperforms positive-only, consistent with self-masking reducing the negative side's effective sample count to $n{-}1$ and thereby increasing its $O(1/n)$ finite-sample bias relative to the positive side.
Third, the fully combined ABC attains the lowest FID across most of training and is never clearly worse than any single-side variant.
Taken together, these observations support adopting the combined correction as our default ABC deployment.

\subsection{Computational Overhead}
\label{sec:overhead}

\begin{wraptable}[7]{r}{0.52\linewidth}
\vspace{-1\baselineskip}
\centering
\caption{Computational overhead comparison.}
\label{tab:overhead}
\small
\begin{tabular}{@{}lcc@{}}
\toprule
 & Standard & ABC \\
\midrule
Peak GPU memory ($\downarrow$) & 4\,729\,MB & 4\,731\,MB\,\upred{2MB} \\
Epoch, no fusion ($\downarrow$) & 37.1\,s & 42.8\,s\,\upred{5.7s} \\
Epoch, fused ($\downarrow$) & 38.0\,s & 37.0\,s\,\dnred{1.0s} \\
Parameters & 26.6\,M & 26.6\,M \\
\bottomrule
\end{tabular}
\end{wraptable}

\Cref{tab:overhead} compares the computational cost of ABC against the standard method at $n{=}8$.
Epoch time is averaged over 10 epochs after startup.
ABC introduces no additional parameters and negligible GPU memory overhead,
since the correction operates entirely on the existing softmax weights via scalar arithmetic.
Without kernel fusion, the per-epoch wall time increases from 37.1\,s to 42.8\,s due to the extra kernel launches for the squared-weight computation and the corrected centroid combination.
However, with compiled execution via \texttt{torch.compile} the added operations are fused by the compiler into the surrounding kernel pipeline, and the wall-time difference from the standard method falls within the measurement noise of our 10-epoch average.
\section{Conclusion}
\label{sec:conclusion}

We identified a structural $O(1/n)$ subsampling bias in drifting models,
arising from the softmax self-normalization of the minibatch estimator.
We further proposed an Analytical Bias Correction (ABC) that reduces the leading-order bias from $O(1/n)$ to $O(1/n^{2})$ while preserving convex-hull containment,
and that can be implemented in two lines of code with negligible wall-time overhead under compiled execution.
Experiments on CIFAR-10 confirmed that ABC delivers lower FID and faster training,
with the largest gains at small~$n$ where the $O(1/n)$ bias is most significant.


\begin{ack}
This work used Anvil \citep{song2022anvil} at Purdue University through allocation CIS251316 from the Advanced Cyberinfrastructure Coordination Ecosystem: Services \& Support (ACCESS) program \citep{boerner2023access}, which is supported by National Science Foundation grants \#2138259, \#2138286, \#2138307, \#2137603, and \#2138296. This research was also supported by the NVIDIA Academic Grant Program using its remote GPU server.
\end{ack}

\bibliographystyle{plainnat}
\bibliography{references}

\newpage
\appendix

\section{Notation Summary}
\label{app:notation}
For ease of reference, we present a table of notations below.
\begin{table}[!ht]
  \caption{Table of notation.}
  \label{tab:notations}
\centering
\small
\begin{tabular}{@{}ll@{}}
\toprule
\textbf{Notation} & \textbf{Description} \\
\midrule
$\bx$ & Query point \\
$\by_i$ & Reference point ($i=1,\dots,n$) \\
$n$ & Minibatch size \\
$w_i = k(\bx,\by_i)$ & Kernel weight \\
$\alpha_i = w_i/\textstyle\sum_j w_j$ & Normalized softmax weight \\
$\Tn = \sum_i \alpha_i\by_i$ & Minibatch centroid (biased) \\
$\thetapop = \EE[w\by]/\EE[w]$ & Target centroid \\
$\Delta_n = \sum_i \alpha_i^2(\Tn - \by_i)$ & Plug-in bias correction \\
$\Tnabc = \Tn - \Delta_n$ & ABC-corrected centroid \\
$\mu_w = \EE[w]$ & Population mean kernel weight \\
$\boldsymbol\mu_v = \EE[w\by]$ & Population weighted mean \\
$\Cov(w,\by) = \EE[(w-\EE w)(\by-\EE \by)]$ & Vector-valued covariance ($\in\RR^D$) \\
$\Var(X) = \EE\|X-\EE X\|^2$ & Variance (trace of covariance for vector $X$) \\
$M,\,C$ & Boundedness constants ($w \leq M$, $\|\by\| \leq C$) \\
$\tau$ & Temperature parameter \\
$D$ & Dimension \\
\bottomrule
\end{tabular}%
\end{table}


\section{Proofs for \texorpdfstring{\Cref{sec:bias}}{Section 3} (The Bias Problem)}
\label{app:proofs_sec3}

\subsection{Proof of \texorpdfstring{\Cref{thm:ratio_bias}}{Theorem 3.3} (Leading-Order Bias)}
\label{app:delta}

\medskip\noindent
\textbf{\Cref{thm:ratio_bias} (Leading-order bias, restated).}
\emph{Let $\by_1,\dots,\by_n\overset{\mathrm{i.i.d.}}{\sim}p$ in $\RR^D$ with
$w=k(\bx,\by)\in[0,M]$ and $\|\by\|\leq C$ almost surely for some constants $M,C>0$, and $\Pr(w>0)=1$.  Then, for fixed $\bx$ with $\mu_w := \EE[w]>0$,}
\begin{equation*}
  \EE[\Tn] - \thetapop
  \;=\; -\,\frac{1}{n}\,
    \frac{\EE\!\bigl[w^2(\by-\thetapop)\bigr]}{\mu_w^2}
  \;+\; O(n^{-2}).
\end{equation*}

\medskip

\begin{proof}
We give the full proof using the multivariate delta method, which
makes the remainder control explicit.

\paragraph{Setup.}
Recall $\Tn = \bar\bv / \bar{w}$, where
$\bar\bv = n^{-1}\sum_{i=1}^n w_i\by_i$ and
$\bar{w} = n^{-1}\sum_{i=1}^n w_i$.
Define the full-distribution expectations $\mu_w := \EE[w]$,
$\boldsymbol\mu_v := \EE[w\by]$, and the target
$\thetapop := \boldsymbol\mu_v / \mu_w$.
For the vector-valued map $g(\bv,w)=\bv/w$, partial derivatives in $\bv$ are linear maps on $\RR^D$ (e.g., $\partial g/\partial \bv$ acts as $(1/\mu_w)\,I_D$), and $\Cov(w\by,w)$ denotes the componentwise vector $\EE[(w\by-\boldsymbol\mu_v)(w-\mu_w)] \in \RR^D$.

\paragraph{Step 1. Define the smooth map and expand.}
Let $g(\bv, w) := \bv / w$, so that $\Tn = g(\bar\bv, \bar{w})$ and
$\thetapop = g(\boldsymbol\mu_v, \mu_w)$.
Taylor-expanding $g$ around $(\boldsymbol\mu_v, \mu_w)$ to second order,
\begin{align}\label{eq:delta_taylor}
  \Tn
  &= g(\boldsymbol\mu_v,\mu_w)
    + \frac{\partial g}{\partial \bv}\,(\bar\bv - \boldsymbol\mu_v)
    + \frac{\partial g}{\partial w}\,(\bar{w} - \mu_w) \notag\\
  &\quad
    + \frac{1}{2}\,\frac{\partial^2 g}{\partial w^2}\,
      (\bar{w} - \mu_w)^2
    + \frac{\partial^2 g}{\partial \bv\,\partial w}\,
      (\bar\bv - \boldsymbol\mu_v)(\bar{w} - \mu_w)
    + R_n,
\end{align}
where $R_n := \Tn - P_2(\bar\bv, \bar{w})$ is the exact second-order Taylor residual, with $P_2$ the degree-$2$ Taylor polynomial of $g$ at $(\boldsymbol\mu_v,\mu_w)$. We show in the ``Remainder bound'' paragraph below that $\|\EE[R_n]\|=O(n^{-2})$ under the stated assumptions.

\paragraph{Step 2. Compute partial derivatives.}
For $g(\bv,w) = \bv/w$, evaluating at $(\boldsymbol\mu_v, \mu_w)$ yields
\begin{equation}\label{eq:partials}
  \frac{\partial g}{\partial \bv} = \frac{I_D}{\mu_w}, \quad
  \frac{\partial g}{\partial w} = -\frac{\boldsymbol\mu_v}{\mu_w^2}
    = -\frac{\thetapop}{\mu_w}, \quad
  \frac{\partial^2 g}{\partial w^2}
    = \frac{2\boldsymbol\mu_v}{\mu_w^3}
    = \frac{2\thetapop}{\mu_w^2}, \quad
  \frac{\partial^2 g}{\partial \bv\,\partial w}
    = -\frac{I_D}{\mu_w^2}.
\end{equation}
Note that $\partial^2 g/\partial \bv^2 = 0$ (linearity in $\bv$), so no quadratic-in-$\bv$ term appears in~\eqref{eq:delta_taylor}.

\paragraph{Step 3. Take expectations.}
The first-order terms vanish since
$\EE[\bar\bv - \boldsymbol\mu_v] = \mathbf{0}$ and
$\EE[\bar{w} - \mu_w] = 0$.
The bias thus comes entirely from the second-order terms,
\begin{equation}\label{eq:delta_bias_raw}
  \EE[\Tn] - \thetapop
  = \frac{\thetapop}{\mu_w^2}\,\EE\bigl[(\bar{w}-\mu_w)^2\bigr]
  - \frac{1}{\mu_w^2}\,
    \EE\bigl[(\bar\bv - \boldsymbol\mu_v)(\bar{w}-\mu_w)\bigr]
  + O(n^{-2}).
\end{equation}

\paragraph{Step 4. Evaluate covariances using the i.i.d.\ structure.}
Since $w_1,\dots,w_n$ are i.i.d., we have
\begin{align}
  \EE\bigl[(\bar{w}-\mu_w)^2\bigr]
    &= \Var(\bar{w})
    = \frac{1}{n}\,\Var(w)
    = \frac{1}{n}\bigl(\EE[w^2]-\mu_w^2\bigr),
    \label{eq:delta_varw}\\[4pt]
  \EE\bigl[(\bar\bv-\boldsymbol\mu_v)(\bar{w}-\mu_w)\bigr]
    &= \frac{1}{n}\,\Cov(w\by,\,w)
    = \frac{1}{n}\bigl(\EE[w^2\by]-\boldsymbol\mu_v\,\mu_w\bigr).
    \label{eq:delta_covvw}
\end{align}
The cross-covariance reduces to a single-sample quantity because for
$i \neq j$ independence gives
$\EE[(w_i\by_i - \boldsymbol\mu_v)(w_j - \mu_w)] = \mathbf{0}$.

\paragraph{Step 5. Simplify to the displacement form.}
Substituting \eqref{eq:delta_varw} and \eqref{eq:delta_covvw} into
\eqref{eq:delta_bias_raw} yields
\begin{align}
  \EE[\Tn] - \thetapop
  &= \frac{1}{n\mu_w^2}\Bigl[
    \thetapop\bigl(\EE[w^2]-\mu_w^2\bigr)
    - \bigl(\EE[w^2\by]-\boldsymbol\mu_v\,\mu_w\bigr)
  \Bigr] + O(n^{-2}) \notag\\
  &= \frac{1}{n\mu_w^2}\Bigl[
    \thetapop\,\EE[w^2] - \EE[w^2\by]
    \;\underbrace{-\;\thetapop\mu_w^2 + \boldsymbol\mu_v\,\mu_w}_{
      =\,0\;\text{since } \boldsymbol\mu_v = \mu_w\thetapop}
  \Bigr] + O(n^{-2}) \notag\\
  &= -\frac{1}{n\mu_w^2}\,\EE\bigl[w^2(\by - \thetapop)\bigr]
    + O(n^{-2}).
    \label{eq:delta_final}
\end{align}
The cancellation in the second line is exact, since
$\boldsymbol\mu_v\,\mu_w - \thetapop\mu_w^2
= \mu_w(\boldsymbol\mu_v - \thetapop\mu_w) = \mathbf{0}$
by the definition $\thetapop = \boldsymbol\mu_v/\mu_w$.

\paragraph{Remainder bound.}
We show $\|\EE[R_n]\|=O(n^{-2})$ by splitting
$\EE[R_n]=\EE[R_n\,\mathbf{1}_{G_n}]+\EE[R_n\,\mathbf{1}_{G_n^c}]$
on the good event $G_n := \{\bar w \geq \mu_w/2\}$.
On $G_n$, the map $g(\bv,w)=\bv/w$ is smooth with uniformly bounded partials, while
$G_n^c$ has super-polynomially small probability.

\emph{Step 1. Dispose of $G_n^c$.}
Since $w\in[0,M]$, Hoeffding's inequality~\citep{hoeffding1963} with
$t=\mu_w/2$ gives
\begin{equation}\label{eq:Gn_tail}
  \Pr(G_n^c)\;\leq\;2\,\exp\!\left(-\,\frac{n\,\mu_w^2}{2M^2}\right),
\end{equation}
which decays faster than any polynomial in $n$.

The assumption $\Pr(w>0)=1$ gives $\bar w>0$ a.s., so $\Tn$ is a convex combination of the $\by_i$ with $\|\Tn\|\leq C$. 
Meanwhile $P_2$ is a fixed polynomial in $(\bar\bv,\bar w)$ with coefficients that are rational functions of $(\boldsymbol\mu_v,\mu_w)$, finite since $\mu_w>0$, evaluated on the bounded region $\|\bar\bv\|\leq MC$, $\bar w\leq M$. 
Therefore both $\Tn$ and $P_2(\bar\bv,\bar w)$ are a.s.\ bounded. Triangle inequality then gives $\|R_n\|\leq \|\Tn\|+\|P_2\|\leq C_0$ a.s.\ for some $C_0(M,C,\mu_w,\|\thetapop\|)$, and therefore
\begin{equation*}
  \|\EE[R_n\,\mathbf{1}_{G_n^c}]\|
  \;\leq\; C_0\,\Pr(G_n^c)
  \;=\; O(n^{-2}).
\end{equation*}

\emph{Step 2. Taylor expansion on $G_n$.}
Let $\boldsymbol Z_i := (w_i\by_i-\boldsymbol\mu_v,\,w_i-\mu_w)\in\RR^{D+1}$ denote the centred $i$-th sample contribution ($\EE[\boldsymbol Z_i]=\mathbf 0$), and write $\boldsymbol{\xi}_n := (\bar\bv-\boldsymbol\mu_v,\,\bar w-\mu_w) = n^{-1}\sum_{i=1}^n \boldsymbol Z_i$, with coordinates indexed by $\ell\in\{1,\dots,D+1\}$.
By Taylor's theorem, expanded to fourth order,
\begin{equation}\label{eq:Rn_expand}
  R_n \;=\; \underbrace{\tfrac{1}{6}\!\sum_{\ell_1,\ell_2,\ell_3}\!
   \partial^3_{\ell_1\ell_2\ell_3} g(\boldsymbol\mu_v,\mu_w)\;
   \xi_n^{(\ell_1)}\xi_n^{(\ell_2)}\xi_n^{(\ell_3)}}_{=:\,R_n^{(3)}}
   \;+\; R_n^{(4)},
\end{equation}
where the third-order partials are evaluated at $(\boldsymbol\mu_v,\mu_w)$, and $R_n^{(4)}$ is the fourth-order remainder.

Every $k$-th partial of $g(\bv,w)=\bv/w$ is a polynomial in $\bv$ divided by $w^{k+1}$, so on $G_n$ (where $\bar w\geq\mu_w/2$ and $\|\bv\|\leq MC$) all third- and fourth-order partials along the segment are uniformly bounded.

\emph{Step 3. Cubic and quartic expectations are $O(n^{-2})$.}
Recall from Step~2 that $\xi_n^{(\ell)} = n^{-1}\sum_{i=1}^n Z_i^{(\ell)}$
with $\EE[Z_i^{(\ell)}]=0$ and bounded fourth moments
(from $w\in[0,M]$ and $\|\by\|\leq C$).
A direct index count then gives the following.
\begin{itemize}\itemsep2pt
  \item \textbf{Cubic.}
    \(\displaystyle
      \EE\!\bigl[\xi_n^{(\ell_1)}\xi_n^{(\ell_2)}\xi_n^{(\ell_3)}\bigr]
      = \frac{1}{n^3}\sum_{i,j,k} \EE\!\bigl[Z_i^{(\ell_1)}Z_j^{(\ell_2)}Z_k^{(\ell_3)}\bigr]
      = \frac{1}{n^3}\cdot n \cdot \EE\!\bigl[Z_1^{(\ell_1)}Z_1^{(\ell_2)}Z_1^{(\ell_3)}\bigr]
      = O(n^{-2})\),
    since terms with any non-matching index factor through
    $\EE[Z_i^{(\ell)}]=0$. Only the $n$ diagonal terms $i=j=k$ survive,
    each a bounded third cross-moment.

    Since $R_n^{(3)}$ has sample-independent coefficients, summing the cubic moments
    over the $(D+1)^3$ index triples gives
    $\|\EE[R_n^{(3)}]\|=O(n^{-2})$. Writing
    $\EE[R_n^{(3)}\,\mathbf{1}_{G_n}]
    =\EE[R_n^{(3)}]-\EE[R_n^{(3)}\,\mathbf{1}_{G_n^c}]$
    and noting that $R_n^{(3)}$ is a.s.\ bounded (a polynomial in the a.s.\ bounded $\boldsymbol\xi_n$), the second term is $O(\Pr(G_n^c))=O(n^{-2})$.
    Hence $\|\EE[R_n^{(3)}\,\mathbf{1}_{G_n}]\|=O(n^{-2})$.
  \item \textbf{Quartic.}
  By a coordinate-wise Lagrange remainder form, each component of $R_n^{(4)}$ equals $\frac{1}{4!}\sum_{\ell_1,\ell_2,\ell_3,\ell_4}
  \partial^4_{\ell_1\ell_2\ell_3\ell_4} g(c_n)\,
  \xi_n^{(\ell_1)}\xi_n^{(\ell_2)}\xi_n^{(\ell_3)}\xi_n^{(\ell_4)}$
  for some $c_n$ (allowed to differ per coordinate) on the segment joining $(\boldsymbol\mu_v,\mu_w)$
  to $(\bar\bv,\bar w)$. On $G_n$, the $w$-coordinate of $c_n$
  is $\geq\mu_w/2$, so the fourth-order partials along the segment are uniformly bounded, giving $\|R_n^{(4)}\mathbf{1}_{G_n}\| = O(\|\boldsymbol\xi_n\|^4)$ pointwise.
  Each fourth-order monomial $\EE[\xi_n^{(\ell_1)}\cdots\xi_n^{(\ell_4)}]$
  is $O(n^{-2})$, since terms with any singleton index vanish by centring
  and the survivors (two matched pairs, $O(n^2)$ choices each $O(n^{-4})$, or all four matched, $n$ choices each $O(n^{-4})$) each contribute $O(n^{-2})$. By Cauchy--Schwarz, $\|\boldsymbol\xi_n\|^4 \leq (D+1)\sum_\ell (\xi_n^{(\ell)})^4$, so $\EE\|\boldsymbol\xi_n\|^4=O(n^{-2})$. Jensen's inequality $\|\EE[X]\|\leq\EE\|X\|$ then gives
  $\|\EE[R_n^{(4)}\mathbf{1}_{G_n}]\|\leq \EE\|R_n^{(4)}\mathbf{1}_{G_n}\| = O(\EE\|\boldsymbol\xi_n\|^4) = O(n^{-2})$.
\end{itemize}

Combining the three bounds above, $\|\EE[R_n]\|\leq C_R/n^2$ for a constant $C_R$ depending on $M$, $C$, $\mu_w$, $D$,
and $\|\thetapop\|$.
\end{proof}


\section{Proofs for \texorpdfstring{\Cref{sec:abc}}{Section 4} (Analytical Bias Correction)}
\label{app:proofs_sec4}

\subsection{Proof of \texorpdfstring{\Cref{thm:abc}}{Theorem 4.1} (Bias Reduction)}
\label{app:abc_bias}

\medskip\noindent
\textbf{\Cref{thm:abc} (Bias reduction, restated).}
\emph{Under the assumptions of \Cref{thm:ratio_bias} (i.e.,
$w\in[0,M]$ a.s., $\|\by\|\leq C$ a.s., $\Pr(w>0)=1$, $\EE[w]>0$):}
\begin{equation*}
  \EE\bigl[\Tnabc - \thetapop\bigr] = O(n^{-2}).
\end{equation*}

\medskip

\begin{proof}
Since $\Tnabc = \Tn - \Delta_n$, it suffices to show
\begin{equation}\label{eq:abc_goal}
  \EE[\Delta_n] \;=\; \bigl(\EE[\Tn] - \thetapop\bigr) + O(n^{-2}).
\end{equation}
Then $\EE[\Tnabc - \thetapop]
= (\EE[\Tn]-\thetapop) - \EE[\Delta_n] = O(n^{-2})$.

All Taylor expansions and delta-method bounds below are made rigorous by the good-event splitting $G_n = \{\bar w \geq \mu_w/2\}$ introduced in the proof of \Cref{thm:ratio_bias}. On $G_n$ the maps $1/\bar w^k$ and $g(u,w)=u/w^2$ are Lipschitz with constants depending only on $M$ and $\mu_w$, so Taylor remainders and delta-method variance bounds hold deterministically. Off $G_n$ the relevant quantities are uniformly bounded, and $\Pr(G_n^c)$ is super-polynomially small by Hoeffding, so off-$G_n$ contributions are absorbed into the $O(n^{-2})$ remainder. We write the expansions pointwise for brevity.

\paragraph{Step 1. Decompose $\Delta_n$.}
Writing $\Tn - \by_i = (\Tn - \thetapop) + (\thetapop - \by_i)$ and
summing with weights $\alpha_i^2$:
\begin{equation}\label{eq:delta_decomp}
  \Delta_n
  \;=\; \underbrace{\sum_{i=1}^n \alpha_i^2(\thetapop - \by_i)}_{(A)}
  \;+\; \underbrace{\Bigl(\sum_{i=1}^n \alpha_i^2\Bigr)(\Tn - \thetapop)}_{(B)}.
\end{equation}
We show $\EE[(A)] = (\EE[\Tn] - \thetapop) + O(n^{-2})$ and
$\EE[(B)] = O(n^{-2})$.

\paragraph{Step 2. Analyze term (A).}
Writing $\alpha_i^2 = w_i^2/(\sum_j w_j)^2 = w_i^2/(n\bar{w})^2$ with
$\bar{w} = n^{-1}\sum_j w_j$:
\[
  (A) \;=\; \frac{1}{\bar{w}^2} \cdot
    \underbrace{\frac{1}{n}
      \cdot \frac{1}{n}\sum_{i=1}^n w_i^2(\thetapop - \by_i)}_{
      =\,(1/n)\,\bar{U}_n},
  \qquad
  \bar{U}_n := \frac{1}{n}\sum_{i=1}^n w_i^2(\thetapop - \by_i).
\]
Taylor-expand $1/\bar{w}^2$ around $\mu_w$ with
$\epsilon := \bar{w} - \mu_w = O_p(n^{-1/2})$ by Chebyshev's inequality:
\[
  \frac{1}{\bar{w}^2}
  \;=\; \frac{1}{\mu_w^2}
    \Bigl(1 - \frac{2\epsilon}{\mu_w} + O_p(\epsilon^2)\Bigr).
\]
Multiplying:
\[
  (A) \;=\; \frac{\bar{U}_n}{n\mu_w^2}
    \;-\; \frac{2\,\bar{U}_n\,\epsilon}{n\mu_w^3}
    \;+\; O_p\!\Bigl(\frac{\|\bar{U}_n\|\,\epsilon^2}{n}\Bigr).
\]

We take expectations term by term.
\begin{itemize}
\item \emph{Leading term.}
  $\EE[\bar{U}_n] = \EE[w^2(\thetapop-\by)]
  = -\EE[w^2(\by-\thetapop)]$.
  So the first term equals
  $-\EE[w^2(\by-\thetapop)]/(n\mu_w^2)$.
  By \Cref{thm:ratio_bias}, this equals $\EE[\Tn] - \thetapop + O(n^{-2})$.
\item \emph{Cross term.}
  Expanding $\bar{U}_n\,\epsilon = n^{-2}\sum_{i,j} w_i^2(\thetapop-\by_i)(w_j-\mu_w)$, the $i\neq j$ terms vanish by independence (second factor is centred), leaving
  $\EE[\bar U_n\,\epsilon] = n^{-1}\,\EE[w^2(\thetapop-\by)(w-\mu_w)] = n^{-1}\,\Cov\bigl(w^2(\thetapop-\by),\,w\bigr) = O(n^{-1})$.
  Contribution $O(n^{-1})/n = O(n^{-2})$.
\item \emph{Higher-order term.}
  Under $w\in[0,M]$ and $\|\by\|\leq C$, $\|\bar{U}_n\|\leq M^2(\|\thetapop\|+C)$ a.s.\ (as $\bar U_n = n^{-1}\sum_i w_i^2(\thetapop-\by_i)$ is bounded uniformly), so
  $\EE[\|\bar{U}_n\|\,\epsilon^2] \leq M^2(\|\thetapop\|+C)\,\EE[\epsilon^2] = O(n^{-1})$ since $\EE[\epsilon^2]=\Var(\bar w)=O(n^{-1})$.
  Contribution $O(n^{-2})$.
\end{itemize}
Combining, $\EE[(A)] = (\EE[\Tn] - \thetapop) + O(n^{-2})$.

\paragraph{Step 3. Analyze term (B).}
Let $A_n := \sum_i\alpha_i^2$ and $B_n := \Tn - \thetapop$.
Decompose the expectation of a product:
\[
  \EE[(B)]
  \;=\; \EE[A_n]\,\EE[B_n] \;+\; \Cov(A_n, B_n).
\]


\emph{Mean-product.}
Since $w_i \leq M$, $w_i^2 \leq M w_i$, hence $\sum_i w_i^2 \leq M \sum_i w_i$. Therefore $A_n = \frac{\sum_i w_i^2}{(\sum_i w_i)^2} \leq \frac{M}{\sum_i w_i}$. On the good event $G_n = \{\bar w \geq \mu_w/2\}$ this gives $A_n \mathbf{1}_{G_n} \leq 2M/(n\mu_w)$. Off $G_n$, $A_n \leq 1$ trivially, and with $\Pr(G_n^c) = O(n^{-k})$ for all $k$ from~\eqref{eq:Gn_tail}, $\EE[A_n] = O(n^{-1})$. By \Cref{thm:ratio_bias}, $\EE[B_n] = O(n^{-1})$. Hence $\EE[A_n]\,\EE[B_n] = O(n^{-2})$.

\emph{Covariance.}
Since $B_n$ is vector-valued, apply Cauchy--Schwarz componentwise:
$\|\Cov(A_n, B_n)\| \leq \sqrt{\Var(A_n)}\,\sqrt{\EE\|B_n-\EE B_n\|^2}$.
Writing $A_n = (1/n)\,g(\bar u_n, \bar w)$ with $\bar u_n = n^{-1}\sum_i w_i^2$ and $g(u,w)=u/w^2$ Lipschitz on $G_n$, the delta method gives $\Var(A_n) = O(n^{-3})$. Likewise $\Tn = \bar\bv/\bar w$ is Lipschitz on $G_n$, so $\EE\|B_n-\EE B_n\|^2 = O(n^{-1})$. Therefore $\|\Cov(A_n, B_n)\| = O(n^{-2})$.

Combining, $\EE[(B)] = O(n^{-2})$.

\paragraph{Step 4. Conclude.}
From Steps 2 and 3,
$\EE[\Delta_n] = \EE[(A)] + \EE[(B)] = (\EE[\Tn] - \thetapop) + O(n^{-2})$,
which is \eqref{eq:abc_goal}.  Therefore
$\EE[\Tnabc - \thetapop] = O(n^{-2})$.
\end{proof}

\subsection{Proof of \texorpdfstring{\Cref{prop:variance}}{Proposition 4.3} (No First-Order Total-Variance Increase)}
\label{app:variance}

\medskip\noindent
\textbf{\Cref{prop:variance} (No first-order total-variance increase, restated).}
\emph{Under the conditions of \Cref{thm:abc} (i.e., $w\in[0,M]$ a.s., $\|\by\|\leq C$ a.s., $\Pr(w>0)=1$, $\EE[w]>0$):}
\begin{equation*}
  \EE\bigl[\|\Tnabc-\EE[\Tnabc]\|^2\bigr]
  \;=\; \EE\bigl[\|\Tn-\EE[\Tn]\|^2\bigr] + O(n^{-3/2})
  \;=\; O(n^{-1}).
\end{equation*}

\medskip

\begin{proof}
In this proof, $\Var(X)$ denotes $\EE\|X - \EE[X]\|^2$. This equals the trace of the covariance matrix for a vector $X$ (equivalently $\sum_d \Var(X_d)$) and the usual variance for a scalar $X$.
We first bound $\Var(\Tn)$, then decompose $\Var(\Tnabc)$ via $\Tnabc = \Tn - \Delta_n$.

\paragraph{Step 0. Baseline variance.}
By the optimality of the mean in $L^2$, $\Var(\Tn)=\EE[\|\Tn-\EE[\Tn]\|^2]\leq\EE[\|\Tn-\thetapop\|^2]$.
It therefore suffices to show $\EE[\|\Tn-\thetapop\|^2]=O(n^{-1})$.

\emph{On $G_n$.}
On $G_n=\{\bar w\geq\mu_w/2\}$, since $\|\bar\bv\|\leq C\bar w\leq CM$ and $\|\boldsymbol\mu_v\|\leq CM$, the Taylor segment from $(\boldsymbol\mu_v,\mu_w)$ to $(\bar\bv,\bar w)$ lies in a region where $g(\bv,w)=\bv/w$ has uniformly bounded second derivatives.
The first-order Taylor expansion around $(\boldsymbol\mu_v,\mu_w)$ gives
\[
  \Tn - \thetapop = \frac{\bar\bv-\boldsymbol\mu_v}{\mu_w} - \frac{\thetapop(\bar w-\mu_w)}{\mu_w} + R_n',
\]
with $\|R_n'\|\leq C_1\bigl(\|\bar\bv-\boldsymbol\mu_v\|+|\bar w-\mu_w|\bigr)^2$ uniformly on $G_n$ for a constant $C_1$ independent of $n$.
Centred i.i.d.\ sample means of bounded random variables have fourth central moments $O(n^{-2})$, so $\EE[(\|\bar\bv-\boldsymbol\mu_v\|+|\bar w-\mu_w|)^4]=O(n^{-2})$ and hence $\EE[\|R_n'\|^2\,\mathbf{1}_{G_n}]=O(n^{-2})$.
Combined with $\EE\|\bar\bv-\boldsymbol\mu_v\|^2,\EE(\bar w-\mu_w)^2=O(n^{-1})$, the bound $\|a+b\|^2\leq 2(\|a\|^2+\|b\|^2)$ yields
\[
  \EE\!\bigl[\|\Tn-\thetapop\|^2\,\mathbf{1}_{G_n}\bigr]=O(n^{-1}).
\]

\emph{Off $G_n$.}
Since $\|\by\|\leq C$ a.s.\ and $\Tn$ is a convex combination of the $\by_i$, $\|\Tn\|\leq C$ and hence $\|\Tn-\thetapop\|\leq C+\|\thetapop\|$.
By~\eqref{eq:Gn_tail}, $\EE[\|\Tn-\thetapop\|^2\,\mathbf{1}_{G_n^c}]\leq(C+\|\thetapop\|)^2\,\Pr(G_n^c)=o(n^{-k})$ for every $k\geq1$.

Combining, $\Var(\Tn)\leq\EE[\|\Tn-\thetapop\|^2]=O(n^{-1})$.

Since $\Tnabc = \Tn - \Delta_n$,
\begin{equation}\label{eq:var_decomp}
  \Var(\Tnabc)
  = \Var(\Tn)
  - 2\,\EE\bigl[\langle \Tn - \EE[\Tn],\;
    \Delta_n - \EE[\Delta_n] \rangle\bigr]
  + \Var(\Delta_n).
\end{equation}
It suffices to show that both $\Var(\Delta_n)$ and the cross term are
$O(n^{-3/2})$.

\paragraph{Step 1. Correction variance.}
Write $A_n := \sum_i \alpha_i^2 = (\sum_i w_i^2)/(\sum_i w_i)^2$, as also defined in the proof of \Cref{thm:abc}.
We first show
\begin{equation}\label{eq:An2_bound}
  \EE[A_n^2]\;=\;O(n^{-2}).
\end{equation}
From the proof of \Cref{thm:abc}, $A_n\mathbf{1}_{G_n}\leq 2M/(n\mu_w)$, so $\EE[A_n^2\,\mathbf{1}_{G_n}]\leq 4M^2/(n^2\mu_w^2) = O(n^{-2})$.
Off $G_n$, $A_n\leq\sum_i\alpha_i=1$, so $\EE[A_n^2\,\mathbf{1}_{G_n^c}]\leq\Pr(G_n^c)=o(n^{-k})$ for any $k$ by~\eqref{eq:Gn_tail}.
Combining yields~\eqref{eq:An2_bound}.

Since $\Tn$ lies in the convex hull of $\{\by_i\}$,
each $\by_i - \Tn$ is bounded by $\|\by_i\|+\|\Tn\|\leq 2C$ almost surely under the assumption $\|\by\|\leq C$.
Applying the triangle inequality to $\Delta_n = \sum_i\alpha_i^2(\Tn-\by_i)$,
\[
  \|\Delta_n\|
  \;\leq\; \sum_i\alpha_i^2\,\|\by_i-\Tn\|
  \;\leq\; 2C\,A_n,
\]
so $\|\Delta_n\|^2\leq 4C^2\,A_n^2$.  Taking expectations and using~\eqref{eq:An2_bound},
\begin{equation}\label{eq:delta_l2}
  \EE\bigl[\|\Delta_n\|^2\bigr]
  \;\leq\; 4C^2\,\EE[A_n^2]
  \;=\; O(n^{-2}).
\end{equation}
Hence $\Var(\Delta_n)\leq\EE[\|\Delta_n\|^2]=O(n^{-2})$.

\paragraph{Step 2. Bound the cross term.}
By the Cauchy--Schwarz inequality,
\begin{align*}
  \bigl|\EE\langle \Tn - \EE[\Tn],\, \Delta_n - \EE[\Delta_n]\rangle\bigr|
  &\;\leq\; \sqrt{\Var(\Tn)\cdot\Var(\Delta_n)} \\
  &\;\leq\; \sqrt{O(n^{-1})\cdot O(n^{-2})} \;=\; O(n^{-3/2}).
\end{align*}

\paragraph{Step 3. Combine.}
Substituting into~\eqref{eq:var_decomp},
\[
  \Var(\Tnabc)
  = \Var(\Tn) + O(n^{-3/2}) + O(n^{-2})
  = \Var(\Tn) + O(n^{-3/2}).
\]
Since $\Var(\Tn) = O(n^{-1})$, the $O(n^{-3/2})$ remainder is of smaller order, confirming that ABC preserves the same $O(n^{-1})$ leading-order variance.
\end{proof}

\paragraph{Intuition.}
The correction $\Delta_n$ has magnitude $O(1/n)$, the same order as
the bias it removes.  The statistical fluctuations of $\Tn$ are
$O(1/\sqrt{n})$, which is a factor of $\sqrt{n}$ larger.  Subtracting
an $O(1/n)$ term from an $O(1/\sqrt{n})$-noisy estimator shifts the
mean (reducing bias) without meaningfully changing the spread
(variance).


\section{Temperature Cannot Substitute for Bias Correction}
\label{app:temperature}

Since both the subsampling bias and high temperature push the centroid
toward the global mean, a natural question is whether simply lowering
$\tau$ can compensate for the bias, removing the need for an analytical
correction.  Let $\Tn(\bx;\tau)$ denote the minibatch centroid
estimator \eqref{eq:centroid} computed at temperature~$\tau$ from $n$
i.i.d.\ samples.  The question asks whether there exists a
$\tau'\neq\tau$ such that
\[
  \EE[\Tn(\bx;\tau')] \;=\; \thetapop(\bx;\tau)
  \qquad \forall\;\bx,
\]
i.e., whether training at a modified temperature can exactly undo the
bias one would incur at the intended temperature.

\subsection{The Compensation Equation}

Write $k_\tau(\bx,\by):=\exp(-\|\bx-\by\|_2/\tau)$, $w=k_\tau(\bx,\by)$, and $\thetapop(\bx;\tau) = \EE_\tau[w\by]/\EE_\tau[w]$ to make the $\tau$-dependence explicit.  By \cref{thm:ratio_bias}, the biased
estimator satisfies
\[
  \EE[\Tn(\bx;\tau)]
  \;=\; \thetapop(\bx;\tau)
  \;-\; \frac{1}{n}\,
    \frac{\EE_\tau[w^2(\by-\thetapop)]}
         {(\EE_\tau[w])^2}
  \;+\; O(n^{-2}).
\]
For a compensating shift $\tau' = \tau + \Delta$ with
$\Delta = O(n^{-1})$, first-order matching
$\EE[\Tn(\bx;\tau')] = \thetapop(\bx;\tau)$ requires
\begin{equation}\label{eq:compensation}
  \partial_\tau\thetapop(\bx;\tau)\;\Delta
  \;=\; \frac{1}{n}\,
    \frac{\EE_\tau[w^2(\by-\thetapop)]}
         {(\EE_\tau[w])^2}
\end{equation}
simultaneously for every~$\bx$.  Since $\Delta$ is a single scalar independent of $\bx$, \eqref{eq:compensation} imposes two conditions in general dimension~$D$:
\begin{enumerate}
\item $\partial_\tau\thetapop(\bx;\tau)$ and $\EE_\tau[w^2(\by-\thetapop)]$ are \emph{collinear} at every~$\bx$ (vacuous when $D=1$);
\item their proportionality constant is $\bx$-independent. Equivalently, when $D=1$, the scalar ratio
\[
  \frac{\EE_\tau[w^2(\by-\thetapop)]}{(\EE_\tau[w])^2\,\partial_\tau\thetapop(\bx;\tau)}
\]
must be the same constant for every~$\bx$.
\end{enumerate}
Both conditions depend on the local density around~$\bx$ and
generically fail, as different query points have different local weight
distributions, so their bias-to-sensitivity ratios differ.

\subsection{Effective Sample Size Cost}

Even if one were to accept an approximate (non-uniform) bias reduction
by lowering~$\tau$, there is a statistical cost.  The effective sample
size of the self-normalized estimator is
\[
  n_{\mathrm{eff}}
  \;=\; \frac{n\,(\EE[w])^2}{\EE[w^2]}
  \;=\; \frac{n}{\lambda(\bx;\tau)},
\]
where $\lambda = \EE[w^2]/(\EE[w])^2 \geq 1$.  Writing $\beta := 1/\tau$ and $\phi(t):=\log\EE[e^{-t\|\bx-\by\|_2}]$ (convex in $t$), one checks $\log\lambda(\beta) = \phi(2\beta)-2\phi(\beta)$, whose derivative $2[\phi'(2\beta)-\phi'(\beta)]\geq 0$ so $\log\lambda$ is monotonically increasing in $1/\tau$. Lowering~$\tau$ therefore always reduces $n_{\mathrm{eff}}$, concentrating the weights on fewer samples and increasing variance.  In contrast, ABC removes the leading-order bias
for a \emph{fixed}~$\tau$ without altering the weights or
sacrificing sample efficiency.

\subsection{Summary}

Temperature adjustment fails to substitute for bias correction
because:
\begin{enumerate}
  \item the compensation equation~\eqref{eq:compensation} requires a
    global alignment that generically fails across query points;
  \item lowering $\tau$ reduces effective sample size, increasing
    variance.
\end{enumerate}
ABC avoids both issues: it corrects for the fixed target
$\thetapop(\bx;\tau)$, operates per-query-point via the local weight
distribution, and is derived from the bias formula rather than a
heuristic parameter adjustment.

\section{Training Algorithm}
\label{app:algorithm}

To better convey the overall training process with ABC, we provide pseudocode in \Cref{alg:abc_training}.

\begin{algorithm}[htb]
\caption{Drifting model training with ABC}
\label{alg:abc_training}
\begin{algorithmic}[1]
\REQUIRE dataset $\mathcal{D}$, one-step generator $g_\theta$, feature encoder $\phi$, per-class batch size $n$, temperature $\tau$
\REPEAT
\STATE Sample a class $c$, draw $n$ real images $I^{+}$ of class $c$ and $n$ noise draws $z$, and set $I \leftarrow g_\theta(z, c)$
\STATE $y^{+} \leftarrow \phi(I^{+})$, $x \leftarrow \phi(I)$ \hfill \mbox{$\triangleright$ negatives $y^{-}$ of $x$ are the other generated features}
\STATE $\alpha^{+} \leftarrow \mathrm{softmax}\big(-\Vert x - y^{+}\Vert_2/\tau\big)$, \; $\alpha^{-} \leftarrow \mathrm{softmax}\big(-\Vert x - y^{-}\Vert_2/\tau\big)$
\STATE $T^{+} \leftarrow \sum_i \alpha_i^{+} y_i^{+}$, \; $T^{-} \leftarrow \sum_i \alpha_i^{-} y_i^{-}$ \hfill $\triangleright$ minibatch centroids
\STATE $T^{+} \leftarrow \big(1 - \sum_i (\alpha_i^{+})^2\big) T^{+} + \sum_i (\alpha_i^{+})^2 y_i^{+}$ \hfill $\triangleright$ ABC
\STATE $T^{-} \leftarrow \big(1 - \sum_i (\alpha_i^{-})^2\big) T^{-} + \sum_i (\alpha_i^{-})^2 y_i^{-}$ \hfill $\triangleright$ ABC
\STATE $V \leftarrow T^{+} - T^{-}$ \hfill $\triangleright$ drifting field at $x$
\STATE $\mathcal{L} \leftarrow \big\Vert x - \mathrm{sg}[x + V] \big\Vert^2$ \hfill $\triangleright$ $\mathrm{sg}$: stop gradient
\STATE Update $\theta$ with $\nabla_\theta \mathcal{L}$
\UNTIL{converged}
\end{algorithmic}
\end{algorithm}


\section{Toy Experiment Details}
\label{app:toy}

\paragraph{Data distribution.}
The reference distribution $p$ is a four-mode isotropic Gaussian in
$\RR^2$ with centers at $(\pm 0.5, \pm 0.5)$ and per-mode standard
deviation $\sigma = 0.1$.  Each mode is selected with equal probability
$1/4$.  We draw a reference pool of $N = 10^6$ points from $p$.

\paragraph{Query points.}
We sample $200$ query points $\bx$ from the same distribution $p$ (so
queries lie near the data modes, mimicking the realistic setting where
generated samples are close to data clusters).

\paragraph{Estimation protocol.}
For each sample size $n \in \{4, 8, 16, 32, 64, 128, 256, 512\}$ and
each query $\bx$, we repeat $200{,}000$ independent trials. In each trial, we draw $n$
points from the pool (with replacement), compute both the standard
centroid $\Tn$ and the ABC-corrected centroid $\Tnabc$,
and accumulate the trial-averaged centroids.  The bias is then measured
as $\|\bar{T}_n - \thetapop\|$, where $\thetapop$ is the exact
target centroid computed from all $N = 10^6$ pool points.

\paragraph{Bias-corrected norm estimator.}
Since $\|\EE[\Tn] - \thetapop\|$ is not directly observable (we
observe $\|\bar{T}_n - \thetapop\|$, which is inflated by Monte Carlo
noise), we apply a standard variance-corrected norm estimator:
$\hat{b}^2 = \|\bar{T}_n - \thetapop\|^2 - \mathrm{tr}(\hat\Sigma)/M$,
where $\hat\Sigma$ is the sample covariance of the $M = 200{,}000$
trial centroids.  We report $\hat{b} = \sqrt{\max(0, \hat{b}^2)}$.
The plotted curve in \Cref{fig:bias_loglog} is then the arithmetic mean of the per-query $\hat{b}$ values over all $200$ sampled queries at each $n$.
Note that $\hat{b}$ is computed pointwise in~$\bx$, so the across-query aggregation is a diagnostic average rather than a uniform-in-$\bx$ statement.

\paragraph{Temperatures.}
We run the experiment at $\tau \in \{0.1, 0.2\}$ to verify that the
$O(1/n)$ and $O(1/n^2)$ scaling holds across different kernel
bandwidths.  All computations use float64 accumulators on GPU for
numerical stability.

\section{Comparison with Classical Bias-Reduction Baselines}
\label{app:baselines}

Beyond ABC, several existing methods target the leading $O(1/n)$ ratio bias of the standard SNIS centroid: the classical jackknife~\citep{quenouille1956}, a bootstrap correction, and the more recent BR-SNIS estimator~\citep{cardoso2022}.
We compare these four corrections, alongside the uncorrected Standard estimator, on the same toy setup as \Cref{sec:bias_validation}.

\paragraph{Baselines.}
Jackknife, Bootstrap, and ABC correct the standard SNIS centroid $\Tn = \sum_i \alpha_i \by_i$ on a fixed minibatch of $n$ samples through different mechanisms, whereas BR-SNIS replaces $\Tn$ with an i-SIR Markov-chain estimator whose initialization bias decays with burn-in.
All four target the leading $O(1/n)$ bias of $\Tn$.

\textbf{Jackknife}~\citep{quenouille1956}: first-order bias correction via leave-one-out recomputation.
The key observation is that if the bias of $\Tn$ is $c/n + O(1/n^2)$, the specific linear combination
$\Tjack = n\, \Tn - (n-1)\,\overline{\Tloo}$, where $\Tloo$ is the SNIS centroid with the $i$-th sample removed, cancels the leading $1/n$ term and leaves an $O(1/n^2)$ residual.
For SNIS specifically the softmax-marginalization identity $\Tloo = (\Tn - \alpha_i \by_i)/(1-\alpha_i)$ lets us compute all $n$ leave-one-out estimates at $O(nD)$ cost, though with a substantially larger constant than the standard estimator.

\textbf{Bootstrap}: Efron-style bias estimation via $B = 100$ bootstrap resamples.
For each draw of $n$ samples we resample $B$ times with replacement, compute an SNIS estimate $\Tboot$ on each replicate, and approximate the bias by $\overline{\Tboot} - \Tn$; the corrected estimator is $\Tbr = 2\,\Tn - \overline{\Tboot}$.
Cost: $O(BnD)$ per estimate.

\textbf{BR-SNIS}~\citep{cardoso2022}: iterated Sampling-Importance Resampling (i-SIR).
Each iteration concatenates the current state with $n{-}1$ fresh reference samples to form a candidate pool, selects the next state with probability proportional to the SNIS weights, and accumulates the SNIS centroid; the final estimator is the average of the SNIS centroids across $K - k_0$ post-burn-in iterations.
Viewed as a Markov-chain procedure, the chain mixes to the target distribution and the bias decays exponentially in the burn-in $k_0$.
We use $K = 10$ and $k_0 = 1$.
Cost: $O(KnD)$ per estimate, but requires $K(n{-}1){+}1$ fresh reference samples per estimate, a substantially larger sample budget than the other methods, which operate on a fixed $n$-sample pool.

\textbf{ABC (ours)}: in contrast to the three baselines above, which rely on resampling or iterative procedures, ABC uses the analytical leading-bias formula from \Cref{thm:ratio_bias} and substitutes minibatch averages for the unknown full-distribution expectations, yielding a closed-form correction
$\Tnabc = (1{-}\sum_i\alpha_i^2)\,\Tn + \sum_i \alpha_i^2 \by_i$ at $O(nD)$ cost (a single additional matmul, same order as the baseline centroid), with no extra samples or resamples required.

\paragraph{Protocol.}
We use a CPU-scale variant of the \Cref{app:toy} setup at $\tau = 0.1$.
The reference pool contains $5{\times}10^4$ points drawn from the four-mode Gaussian, and we evaluate on $200$ Gaussian query points ($\mathcal{N}(0, 0.25\,I_2)$).
For each $n \in \{4, 8, 16, 32, 64, 128, 256\}$ we run $5{,}000$ Monte Carlo trials per query, drawing $n$ samples without replacement from the pool each trial, and report the per-query bias norm $\|\EE[\Tn] - \thetapop\|$ and total variance $\sum_d \mathrm{Var}(\Tnd)$ averaged across queries.
Unlike \Cref{app:toy}, here we use the naive norm $\|\bar{T}_n - \thetapop\|$ rather than the variance-corrected estimator $\hat b$, so the reported bias values are slightly inflated by Monte Carlo noise (most noticeable for the smallest biases, e.g., BR-SNIS at large $n$).
Per-estimate wall time is the mean across trials (excluding the first $10$ warmup measurements), measured on CPU with PyTorch 2.3.

\begin{table}[t]
\centering
\caption{Bias, variance, and wall time at $n \in \{8, 32, 128\}$ ($\tau=0.1$, $5{,}000$ trials).
BR-SNIS uses $K(n{-}1){+}1$ fresh samples per estimate (with $K=10$); the other methods use $n$.}
\label{tab:baselines}
\small
\begin{tabular}{@{}llccc@{}}
\toprule
$n$ & Method & $\|\mathrm{bias}\|$ ($\downarrow$) & Var ($\downarrow$) & Time, \textmu s ($\downarrow$) \\
\midrule
\multirow{5}{*}{8}
  & Standard         & $6.24 \times 10^{-2}$          & $7.54 \times 10^{-2}$          & $\mathbf{5.9}$ \\
  & Jackknife        & $7.02 \times 10^{-2}$          & $1.95 \times 10^{-1}$          & 82.9 \\
  & Bootstrap        & $3.27 \times 10^{-2}$          & $1.08 \times 10^{-1}$          & 996.6 \\
  & BR-SNIS (i-SIR)  & $\mathbf{7.89 \times 10^{-3}}$ & $\mathbf{6.68 \times 10^{-3}}$ & 1527.6 \\
  & ABC (ours)       & $4.99 \times 10^{-2}$          & $8.03 \times 10^{-2}$          & 25.5 \\
\midrule
\multirow{5}{*}{32}
  & Standard         & $1.33 \times 10^{-2}$          & $8.07 \times 10^{-3}$          & $\mathbf{11.2}$ \\
  & Jackknife        & $3.40 \times 10^{-3}$          & $9.85 \times 10^{-3}$          & 136.0 \\
  & Bootstrap        & $7.07 \times 10^{-3}$          & $9.00 \times 10^{-3}$          & 1985.1 \\
  & BR-SNIS (i-SIR)  & $\mathbf{1.09 \times 10^{-3}}$ & $\mathbf{1.24 \times 10^{-3}}$ & 3521.9 \\
  & ABC (ours)       & $7.26 \times 10^{-3}$          & $8.61 \times 10^{-3}$          & 36.9 \\
\midrule
\multirow{5}{*}{128}
  & Standard         & $3.63 \times 10^{-3}$          & $1.99 \times 10^{-3}$          & $\mathbf{21.3}$ \\
  & Jackknife        & $6.20 \times 10^{-4}$          & $2.15 \times 10^{-3}$          & 325.2 \\
  & Bootstrap        & $9.94 \times 10^{-4}$          & $2.12 \times 10^{-3}$          & 3669.7 \\
  & BR-SNIS (i-SIR)  & $\mathbf{2.18 \times 10^{-4}}$ & $\mathbf{2.60 \times 10^{-4}}$ & 10185.4 \\
  & ABC (ours)       & $1.03 \times 10^{-3}$          & $2.10 \times 10^{-3}$          & 57.0 \\
\bottomrule
\end{tabular}
\end{table}

\paragraph{Findings.}
Among the fixed-$n$-pool methods, ABC attains bias comparable to Jackknife and Bootstrap with variance comparable to the Standard estimator's (which Jackknife notably inflates at small $n$) and running several times faster than either baseline, with no additional samples or resamples.
BR-SNIS attains the smallest bias but at a $K{\times}$-larger sample budget and much higher wall time.

\section{Hyperparameters and Compute}
\label{app:hyperparams}

\paragraph{Architecture.}
We use DriftDiT-Small from \citet{deng2026drifting}: a 26.6M-parameter
diffusion transformer with a frozen DINOv2 multi-scale encoder (four
layers) and bidirectional kernel normalization.  Input resolution is
$32{\times}32$.

\paragraph{Training hyperparameters.}
\Cref{tab:hyperparams} summarizes the training configuration.  All
hyperparameters except training duration follow~\citet{deng2026drifting};
training duration is scaled inversely with $n$ so that every
configuration sees the same total of 6.24M positive samples per class.

\begin{table}[h]
\centering
\caption{Training hyperparameters.}
\label{tab:hyperparams}
\begin{tabular}{@{}ll@{}}
\toprule
Hyperparameter & Value \\
\midrule
Optimizer & AdamW ($\beta_1{=}0.9$, $\beta_2{=}0.95$) \\
Weight decay & $0.01$ \\
Learning rate & $2 \times 10^{-4}$ \\
Warmup & $2\,000$ linear steps \\
EMA decay & $0.999$ \\
Drift-field temperatures $\tau$ & $\{0.02,\,0.05,\,0.2\}$ \\
Steps per epoch & $195$ \\
Epochs ($n{=}64 / 32 / 16 / 8$) & $500 / 1\,000 / 2\,000 / 4\,000$ \\
Total positive samples per class seen & $6.24$M \\
Seeds & $\{42, 43, 44\}$ \\
\bottomrule
\end{tabular}
\end{table}

\paragraph{Evaluation.}
FID is computed on $50\,000$ generated samples per checkpoint.
Evaluation frequency during training ranges from every $25$ to every $200$ epochs depending on configuration.
The smoothed trajectories in \Cref{fig:convergence} use a rolling mean over $\{7, 5, 3, 3\}$ evaluation checkpoints for $n\in\{8, 16, 32, 64\}$ respectively (smaller-$n$ panels run more epochs and are smoothed more aggressively), with bands showing mean $\pm$ max-deviation across the three seeds.

\paragraph{Hardware and software.}
All experiments run on a single NVIDIA A100 80\,GB GPU with PyTorch
and CUDA.  Average per-epoch wall time is approximately $37.1$\,s for
Standard and $42.8$\,s for ABC, with peak GPU memory of
$4\,729$ vs $4\,731$\,MB (\Cref{tab:overhead}).
Total compute for the main CIFAR-10 grid (4 sample sizes $\times$ 2 methods $\times$ 3 seeds $=$ 24 runs) is approximately $700$ A100-hours of single-GPU-equivalent time, plus additional time for ablations and toy experiments.

\section{Additional Results on CIFAR-100}
\label{app:cifar100}

As a consistency check on a harder dataset, we run single-seed ABC vs Standard on CIFAR-100 using the same DriftDiT-Small architecture and hyperparameters as \Cref{sec:setup}, except for the training epochs, which are reduced because of the higher per-epoch cost on CIFAR-100.
Compute parity with the CIFAR-10 budget is maintained for $n{=}16$ and $n{=}32$, while $n{=}8$ is budget-reduced and used only as a preliminary check.
The results show that ABC also improves both last-5 average and best FID on CIFAR-100.

\begin{table}[h]
\centering
\caption{FID on CIFAR-100.}
\label{tab:cifar100}
\begin{tabular}{@{}cccccc@{}}
\toprule
\multirow{2}{*}{$n$} & \multirow{2}{*}{Epochs} & \multicolumn{2}{c}{Last-5 average FID ($\downarrow$)} & \multicolumn{2}{c}{Best FID ($\downarrow$)} \\
\cmidrule(lr){3-4} \cmidrule(lr){5-6}
 &  & Standard & ABC & Standard & ABC \\
\midrule
$8$  & $1\,000$ & $4.17$ & $\mathbf{4.05}$\,\dnred{0.12} & $4.06$ & $\mathbf{3.74}$\,\dnred{0.32} \\
$16$         & $2\,000$ & $2.64$ & $\mathbf{2.52}$\,\dnred{0.12} & $2.52$ & $\mathbf{2.40}$\,\dnred{0.12} \\
$32$ & $1\,000$ & $2.26$ & $\mathbf{2.17}$\,\dnred{0.09} & $2.20$ & $\mathbf{2.10}$\,\dnred{0.10} \\
\bottomrule
\end{tabular}
\end{table}

\section{Precision and Recall}
\label{app:prdc}

To complement the FID metric used in the main text, we additionally report the improved precision and recall of \citet{kynkaanniemi2019improved} on CIFAR-10.
For each run, we take the EMA checkpoint with the lowest FID, generate 50,000 samples with the same sampling configuration as the FID evaluation, and compute both metrics on InceptionV3 pool3 features against the CIFAR-10 training set.
We report both $k{=}3$ (the setting of \citet{kynkaanniemi2019improved}) and $k{=}5$ (the setting in the usage example of the \texttt{prdc} library\footnote{\url{https://github.com/clovaai/generative-evaluation-prdc}}).
\Cref{tab:prdc} reports the results.

\begin{table}[htb]
\centering
\caption{Improved precision and recall on CIFAR-10.}
\label{tab:prdc}
\small
\begin{tabular}{@{}ccccccccc@{}}
\toprule
\multirow{2}{*}{$n$} & \multicolumn{2}{c}{Recall, $k{=}3$ ($\uparrow$)} & \multicolumn{2}{c}{Recall, $k{=}5$ ($\uparrow$)} & \multicolumn{2}{c}{Precision, $k{=}3$ ($\uparrow$)} & \multicolumn{2}{c}{Precision, $k{=}5$ ($\uparrow$)} \\
\cmidrule(lr){2-3} \cmidrule(lr){4-5} \cmidrule(lr){6-7} \cmidrule(lr){8-9}
 & Standard & ABC & Standard & ABC & Standard & ABC & Standard & ABC \\
\midrule
$8$ & $0.26$ & $\mathbf{0.33}$\,\upred{0.07} & $0.36$ & $\mathbf{0.44}$\,\upred{0.08} & $0.64$ & $\mathbf{0.65}$\,\upred{0.01} & $0.75$ & $0.75$ \\
$16$ & $0.39$ & $\mathbf{0.41}$\,\upred{0.02} & $0.50$ & $\mathbf{0.53}$\,\upred{0.03} & $\mathbf{0.67}$ & $0.66$\,\dnred{0.01} & $\mathbf{0.77}$ & $0.76$\,\dnred{0.01} \\
$32$ & $0.47$ & $\mathbf{0.49}$\,\upred{0.02} & $0.58$ & $\mathbf{0.60}$\,\upred{0.02} & $\mathbf{0.66}$ & $0.65$\,\dnred{0.01} & $\mathbf{0.76}$ & $0.75$\,\dnred{0.01} \\
$64$ & $0.48$ & $0.48$ & $0.59$ & $0.59$ & $0.66$ & $0.66$ & $0.76$ & $0.76$ \\
\bottomrule
\end{tabular}
\end{table}

It is clear that ABC improves recall rather than precision. This is consistent with what our theory predicts. \Cref{thm:ratio_bias} shows the bias is a contraction of the centroids of the drifting field, which suppresses diversity, and ABC removes precisely this contraction. The correction is therefore expected to recover recall and leave precision untouched, exactly as observed. Also, the recall gain from ABC is large at $n{=}8$, shrinks through $n{=}16$ and $n{=}32$, and vanishes at $n{=}64$, which is consistent with the $O(1/n)$ trend. These results provide an independent signal beyond FID that the reported improvements reflect the predicted mechanism of bias correction.

\section{Shrinkage-Only Ablation}
\label{app:shrinkage}

To further confirm that the gains come from bias correction, we run an ablation that separates the bias correction from the shrinkage factor. We implement a shrinkage-only variant that applies the damping factor while dropping the corrective term, i.e., $\boldsymbol{T_n^{\mathrm{shrink}}} = (1 - \sum_i \alpha_i^2)\,\Tn$, and report the mean FID trend during training over three seeds in \Cref{fig:shrinkage}.

\begin{figure}[htb]
\centering
\includegraphics[width=\textwidth]{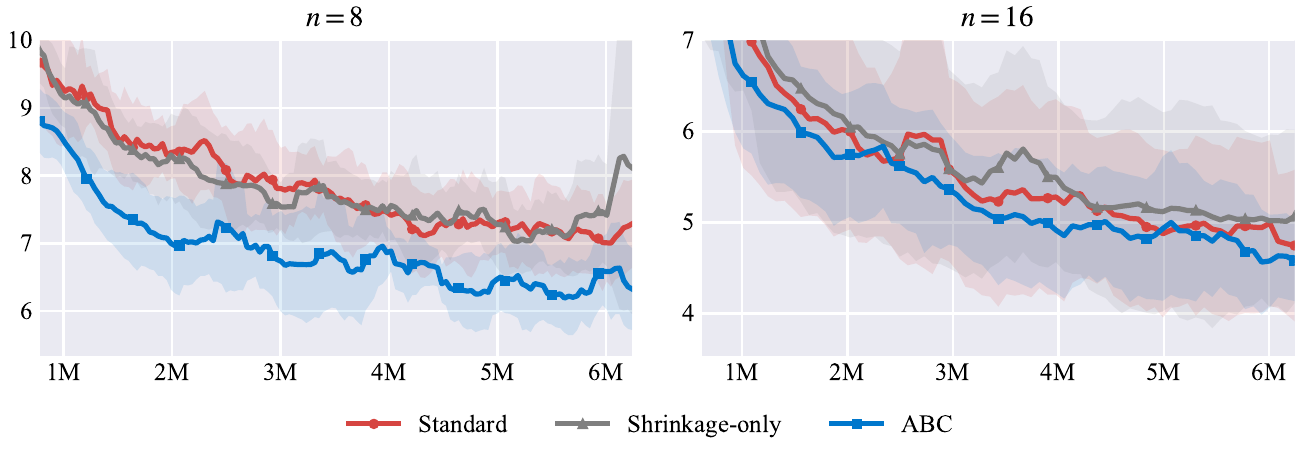}
\caption{FID trajectories at $n{=}8$ and $n{=}16$ for Standard, shrinkage-only, and ABC. Bands show mean $\pm$ max-deviation across three seeds.}
\label{fig:shrinkage}
\end{figure}


These results match the theoretical prediction. The corrective term that the shrinkage-only variant drops is itself of order $O(1/n)$, so its residual bias remains $O(1/n)$, the same order as Standard. Consistently, the shrinkage-only variant stays close to or even worse than Standard at both $n$, and is clearly behind ABC throughout. This supports the conclusion that ABC's improvement comes from removing the bias rather than from shrinkage or variance damping.

\section{Experiments with \texorpdfstring{$n{=}128$}{n=128}}
\label{app:n128}

We additionally train Standard and ABC on CIFAR-10 at $n{=}128$ with a single seed, using the same total optimizer steps as the $n{=}64$ runs (500 epochs). Over the last five evaluations, the average FID is $3.08$ for Standard and $2.92$ for ABC, and the best FID is $2.88$ and $2.68$, respectively. The results are consistent with the $O(1/n)$ bias, while ABC still retains a small advantage.

\section{Limitations and Future Work}
\label{app:limitations}

The primary contribution of this paper is the theoretical analysis of the $O(1/n)$ subsampling bias in drifting models and the analytical correction (ABC) that addresses it.  Our experiments validate this theory rather than pursue state-of-the-art generation quality. Compute constraints also limit the experimental scope.  Follow-up and application-oriented work could further tune the temperature $\tau$ and other hyperparameters in the presence of ABC to achieve better task-specific performance.

\clearpage


\end{document}